\pdfoutput=1
\documentclass[conference]{IEEEtran}
\IEEEoverridecommandlockouts
\usepackage{cite}
\usepackage{amsmath,amssymb,amsfonts}
\usepackage{graphicx}
\usepackage{textcomp}
\graphicspath{{./figures/}}
\usepackage{subcaption}
\usepackage{multirow}
\usepackage{multicol}
\usepackage{array}
\usepackage{float}
\usepackage{amsmath,amsfonts}
\usepackage{amsthm}
\usepackage{mathrsfs}
\usepackage{bm}
\usepackage{bbm}
\usepackage{dsfont}
\usepackage{algorithm}
\usepackage{algorithmic}

\newtheorem{lemma}{Lemma}
\newtheorem{definition}{Definition}

\newtheorem{proposition}{Proposition}

\usepackage{xcolor}
\usepackage{booktabs}

\def\BibTeX{{\rm B\kern-.05em{\sc i\kern-.025em b}\kern-.08em
    T\kern-.1667em\lower.7ex\hbox{E}\kern-.125emX}}
    
\makeatletter
\newcommand{\linebreakand}{%
  \end{@IEEEauthorhalign}
  \hfill\mbox{}\par
  \mbox{}\hfill\begin{@IEEEauthorhalign}
}
\makeatother

\begin{document}
\pagestyle{plain}
\title{Towards Higher-order Topological Consistency for Unsupervised Network Alignment}


\author{\IEEEauthorblockN{1\textsuperscript{st} Qingqiang Sun}
\IEEEauthorblockA{
\textit{University of New South Wales}\\
 Sydney, Australia \\
qingqiang.sun@unsw.edu.au}
\and
\IEEEauthorblockN{2\textsuperscript{nd} Xuemin Lin}
\IEEEauthorblockA{
\textit{Shanghai Jiao Tong Universiy}\\
Shanghai, China \\
xuemin.lin@gmail.com}
\and
\IEEEauthorblockN{3\textsuperscript{rd} Ying Zhang}
\IEEEauthorblockA{
\textit{University of Technology Sydney}\\
Sydney, Australia \\
ying.zhang@uts.edu.au}
\linebreakand 
\IEEEauthorblockN{4\textsuperscript{th} Wenjie Zhang\thanks{\IEEEauthorrefmark{1}Corresponding author.}\IEEEauthorrefmark{1}}
\IEEEauthorblockA{
\textit{University of New South Wales}\\
Sydney, Australia \\
zhangw@cse.unsw.edu.au}
\and
\IEEEauthorblockN{5\textsuperscript{th} Chaoqi Chen}
\IEEEauthorblockA{
\textit{University of Hong Kong}\\
Hong Kong, China \\
chencq@connect.hku.hk}
}

\maketitle

\begin{abstract}
Network alignment task, which aims to identify corresponding nodes in different networks, is of great significance for many subsequent applications. Without the need for labeled anchor links, unsupervised alignment methods have been attracting more and more attention. However, the topological consistency assumptions defined by existing methods are generally low-order and less accurate because only the edge-indiscriminative topological pattern is considered, which is especially risky in an unsupervised setting. To reposition the focus of the alignment process from low-order to higher-order topological consistency, in this paper, we propose a fully unsupervised network alignment framework named HTC. The proposed higher-order topological consistency is formulated based on edge orbits, which is merged into the information aggregation process of a graph convolutional network so that the alignment consistencies are transformed into the similarity of node embeddings. Furthermore, the encoder is trained to be multi-orbit-aware and then be refined to identify more trusted anchor links. Node correspondence is comprehensively evaluated by integrating all different orders of consistency. {In addition to sound theoretical analysis, the superiority of the proposed method is also empirically demonstrated through extensive experimental evaluation. On three pairs of real-world datasets and two pairs of synthetic datasets, our HTC consistently outperforms a wide variety of unsupervised and supervised methods with the least or comparable time consumption. It also exhibits robustness to structural noise as a result of our multi-orbit-aware training mechanism.}
\end{abstract}

\begin{IEEEkeywords}
unsupervised network alignment, edge orbit, graph convolutional network, high-order topological consistency
\end{IEEEkeywords}

\section{Introduction}
Network alignment task, which aims to identify entity correspondence across different networks, is usually the very first step of many downstream analyzing tasks. For instance, recognizing the same user on different social networks can facilitate friend suggestion, item recommendation, personalized advertisement  \cite{dong2012link,xiang2018online, li2014matching, liu2016aligning, yin2016adapting}. Similar scenarios also exist widely in other fields, such as protein network analysis \cite{nassar2018low}, knowledge discovery \cite{zhan2015influence}, etc. 

Identifying corresponding nodes across different networks is an extremely hard task, even for humans. Manually labelling correspondence can be prohibitively challenging, expensive (in human efforts, time, and money costs), and tedious \cite{Zhang2019}. Due to such obstacles, in some cases, it may be impractical to get access to sufficient labels for training well-performed supervised or even semi-supervised models \cite{mu2016user, liu2016aligning}. By contrast, unsupervised models can be trained without the need for labeled data, which is more flexible and practical in real-world application scenarios. Thus, unsupervised alignment methods have been drawing a surge of interest recently \cite{trung2020adaptive, zhu2020huna, zhou2020unsupervised}.

In general, most existing alignment methods construct their models based on the assumption of attribute consistency or structural consistency, no matter supervised or unsupervised \cite{trung2020comparative}. For instance, IsoRank~\cite{singh2008global} is a typical topology-only method that recognizes two nodes from two networks to be similar if their neighborhoods are similar. The most direct application of attribute consistency is to heuristically align users with the same user profiles such as name, gender, birthday, and location \cite{labitzke2011your, liu2013s}. Instead of relying only on either attribute consistency or structural consistency, an increasing body of recent work makes efforts to take both attributes and structural information into consideration (if both are available) and receives significant performance gains by embracing the complementary effect between them \cite{zhang2016final, heimann2018regal, trung2020adaptive}.

Despite various types of topological information being utilized, such as node degree~\cite{heimann2018regal}, connectivity~\cite{man2016predict}, and co-occurrence in random walks~\cite{zhou2018deeplink}, most existing works extract such information by relying solely on the most trivial topological pattern and accordingly define their low-order topological consistency. Specifically, in the \emph{trivial} perspective of network topology, the connections between nodes are indiscriminative; that is, each edge plays the same role in the network. Nevertheless, edges can be further distinguished in a \emph{higher-order} view of network topology based on their specific functions in local network structures, such as motifs \cite{milo2002network} or graphlets \cite{prvzulj2004modeling} (it is worth noting that the so-called higher-order topology is not simply equivalent to a larger neighborhood). The main differences between trivial and high-order topology lie in two ways:

(a) \textit{Topological discrimination.} Compared with trivial topology, high-order topology is more discriminative. For instance, as displayed in Fig.~\ref{fig: orbit network}, all edges have the same weights/widths in the plain type of network topology, but they show a significant difference when higher-order edge patterns are considered (details about edge orbit will be illustrated in section \ref{sec: methodology}). As a result, the local topology around a node can be distinguished to a larger extent.  

(b) \textit{Accuracy in topological consistency.} As the trivial topology is of less discrimination, the corresponding topological consistency defined by that may not be accurate enough and hence degrades the alignment performance. 
For example, say we want to identify the anchor node of node \textit{A} in Fig.~\ref{fig: orbit network} based on their {one-hop neighborhood}. Any node with the same {degree} as node \textit{A} can be regarded as satisfying low-order topological consistency based on the original topological pattern, which may be easily satisfied by a lot of candidates and lead to an ambiguous alignment, while in the higher-order topology, a potential anchor node should not only have an identical {degree} but also have edges with similar {weights} connecting its neighbors, which results in a more accurate consistency and helps to filter out deceptive candidates. 

Focusing solely on the trivial topology and corresponding consistency is generally more risky in the unsupervised setting. The fact is, models lack the capacity to extract extra information favourable to alignment accuracy from the trivial topology without the guidance of label information, and hence heavily depend on the consistency assumption. In other words, developing low-order topological consistency to high order is more crucial in this scenario.

\begin{figure}[t]
    \setlength{\abovecaptionskip}{-0.0cm}   
    \setlength{\belowcaptionskip}{-0.3cm}  
    \centering
    \includegraphics[width=0.47\textwidth]{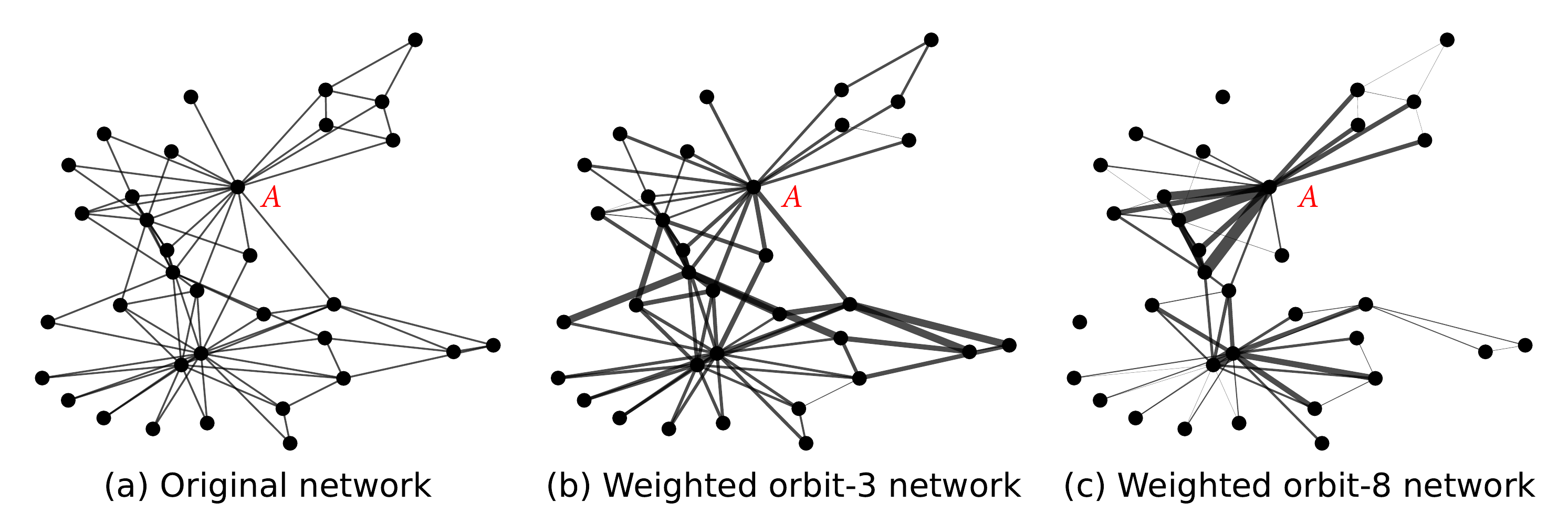}
    \caption{Network topology can differ significantly when we define adjacency based on different orbit patterns. Take node \textit{A} as an example: all of its edges are identical in (a); its edges in (b) and (c) have different weights/widths according to the frequency that they occur on orbit 3 and 8 respectively (the edge removed is equivalent to its corresponding weight being 0).} 
    \label{fig: orbit network}
\end{figure}

To alleviate the aforementioned issues, in this paper, we propose a fully unsupervised framework called \textbf{HTC} (towards \underline{H}igher-order \underline{T}opological \underline{C}onsistency for unsupervised network alignment). Our assumption is that two nodes are more likely to have an anchor link if they hold not only low-order but also higher-order topological consistency. Considering that, we define higher-order topological consistency based on edge orbits. By conducting an orbit-weighted aggregation process, we transform the higher-order topological consistency and attribute consistency into the similarity of corresponding node embeddings, which is theoretically analyzed and proved. Without labeled data, we learn from the training paradigm of Graph Auto-Encoder (GAE)~\cite{kipf2016variational} but share the encoding parameters between different orders of topology so that the encoder can be multi-orbit-aware. In addition, we propose an effective fine-tuning mechanism to find more trustworthy anchor links by assigning trusted node pairs a larger aggregation coefficient, and its validity is verified in theory. A comprehensive evaluation of node correspondence is obtained by further integrating alignment scores computed based on different orders of consistency, rather than focusing on any single topological pattern. 

{The proposed HTC framework is not only theoretically sound but also empirically performs well. HTC consistently outperforms a wide variety of state-of-the-art models, even though some of those competitors are supervised, across three pairs of real-world datasets with the least or comparable time cost. In particular, significant improvements (up to 49\% in terms of $precision@1$) over the basic Graph Convolutional Network (GCN) based alignment paradigm are observed continuously. Experimental results on two pairs of synthetic datasets also verify the robustness of HTC against structural noise. Other experimental results, including ablation tests, hyperparameter sutdy, and visulization analysis, further verify the significance of the proposed higher-order topological consistency for improving alignment precision.}

Our contributions can be summarized as follows:   
\begin{itemize}
    \item We propose a fully unsupervised network alignment framework, which formulates higher-order topological consistency based on edge orbits.
    \item It is theoretically proved that the formulated higher-order topological consistency combined with attribute consistency can derive node embedding similarity, which transforms the network alignment task into the task of node embedding similarity measurement.
    {\item As an additional effect of the proposed multi-orbit-aware training mechanism, our method can be robust to structural noise.}
   {\item We introduce the concept of trusted pairs and accordingly refine embeddings so as to find more trusted pairs, which can alleviate the hubness problem that accompanies the roughly-learned embeddings.}
    {\item The superiority of our method, such as effectiveness, efficiency, and robustness, is comprehensively evaluated through extensive experiments.}
\end{itemize}

The rest of this paper is organized as follows. We briefly review related work in Section \ref{sec: related work}. The problem is formulated in Section \ref{sec: formulation} before proposing our framework in Section \ref{sec: methodology}. Section \ref{sec: experiment} reports experimental results, followed by conclusions made in Section \ref{sec: conclusion}.

\section{Related Work} \label{sec: related work}
\subsection{Network Alignment}
So far, most network alignment approaches are supervised. Supervisory data are required so as to refine embedding space either in the form of partial ground truth (e.g. PALE \cite{man2016predict}, CENALP \cite{du2019joint}) or prior alignment matrix (e.g. IsoRank \cite{singh2008global}, FINAL \cite{zhang2016final}). Due to the expensive cost of manual labeling, unsupervised methods are more suitable and desirable in some real-life applications. 

Another noticeable trend is that more and more methods try to leverage both structural and attribute information so as to better capture the information of the network nodes (such as FINAL\cite{zhang2016final} and REGAL~\cite{heimann2018regal}). Recently, some work has taken advantage of the natural strength of GCN in integrating both network topology and node attributes and has empirically shown the great potential of GCN for alignment tasks \cite{trung2020adaptive, liang2021unsupervised}. A wide variety of structural metrics are proposed and utilized for formulating corresponding topological consistency, like node degree~\cite{heimann2018regal}, connectivity~\cite{man2016predict}, and co-occurrence in random walks~\cite{zhou2018deeplink}. Besides, some methods pay attention to a broader neighborhood of nodes, such as stacking more GCN layers~\cite{trung2020adaptive} or aligning triangular structures~\cite{mohammadi2016triangular}. Even so, most previous work essentially just models the edge-indiscriminative topological pattern, and hence their corresponding topological consistencies remain low-order. Unlike previous work, we formulate the high-order topological consistency based on different orders of edge pattern. It should be emphasized that larger neighborhoods and higher-order consistency cannot be simply equalized. The experimental results also empirically demonstrate that our proposed higher-order topological consistency performs far better than simply considering a larger neighborhood, such as introducing diffusion matrices\cite{klicpera2019diffusion}.

\subsection{High-order Pattern}
High-order patterns of network structure such as graphlets and motifs have been demonstrated to be very informative and helpful for different tasks in many fields. For instance, graph motifs are utilized in \cite{monti2018motifnet} for classifying directed CORA citation networks. Graphlet patterns are taken into consideration to obtain better performance in the semi-supervised node classification task \cite{lee2019graph}. In \cite{sankar2019meta}, the use of motifs enables the perception of high-order semantics in heterogeneous graphs. In more detail, each node or edge can be differentiated at the level of orbit \cite{prvzulj2007biological, solava2012graphlet}. 
{Some network alignment techniques, such as H-GRAAL~\cite{milenkovic2010optimal}, GREAT~\cite{crawford2015great}, GraphletAlign~\cite{almulhim2019network}, etc., treat the graph degree vectors as node/edge features and then quantify the corresponding topological similarity for network alignment. By contrast, we define high-order topological consistency based on the constructed graph orbit matrix, which is injected into the aggregation process of GCN in order to further extract and integrate useful information. In this way, we are capable of obtaining abstract features of nodes for similarity computation.}



\subsection{Graph Convolutional Network}
Compared with other node embedding techniques like Node2Vec \cite{grover2016node2vec}, DeepWalk~\cite{perozzi2014deepwalk}, and LINE~\cite{tang2015line}, GCN \cite{kipf2016semi} is naturally capable of integrating both structure and attribute information due to its special encoding mechanism and has shown its effectiveness in processing graph-structured data \cite{nguyen2018graph}. To address the issue that GCN treats every edge equally without distinction, Graph Attention Network and its variants are proposed, which generally need labels to train extra parameters \cite{velivckovic2017graph, lee2019graph}. However, by using higher-order topological patterns, we can distinguish edges without introducing additional training-required parameters, which is a great merit for unsupervised learning.

Auto-Encoder and its variants have important applications in many fields. For example, Wang et al.~\cite{wang2021deep} use variational autoencoder with dynamic extensions to successfully address the dynamic interaction problem of latent variables. As a special variant, Graph Autoencoder (GAE) \cite{kipf2016variational} is proposed for unsupervised learning on graph-structured data. Inspired by such an encoder-decoder paradigm, we train our GCN encoder to be multi-orbit-aware by sharing encoding parameters among different orders of orbit topology. Hence, it can be expected to prevent the model from relying on any single topological pattern.

\section{Problem Formulation} \label{sec: formulation}
Our work targets the problem of aligning two \textit{attributed networks}. Since the alignment tasks are usually conducted on networks in similar domains, it is easy to find common (not necessarily all) attributes in their attribute spaces. Generally, an attributed network can be represented as $G=(\mathcal{V}, \mathbf{A}, \mathbf{X})$, where $\mathcal{V}$ is a set consisting of $n$ nodes, $\mathbf{A}\in(0,1)^{n\times n}$ is the adjacency matrix of the network which contains the connection details among those nodes, and $\mathbf{X}\in \mathbb{R}^{n\times d}$ is the attribute/feature matrix in which every node has a $d$-dimensional feature. 

Typically, one of two to-be-alined networks/graphs is called \textit{source network/graph} and the other is called \textit{target network/graph}, denoted as $G_s=(\mathcal{V}_s, \mathbf{A}_s, \mathbf{X}_s)$ and $G_t = (\mathcal{V}_t, \mathbf{A}_t, \mathbf{X}_t)$ respectively. The correspondence is referred to as \textit{anchor links}, and the nodes forming the links are referred to as \textit{anchor nodes}. Mathematically, the unsupervised network alignment task can be defined as follows:
\begin{definition}  
\textbf{Unsupervised Network Alignment.}

 Given a source network $G_s=(\mathcal{V}_s, \mathbf{A}_s, \mathbf{X}_s)$ and a target network $G_t = (\mathcal{V}_t, \mathbf{A}_t, \mathbf{X}_t)$,  the objective of unsupervised network alignment is to identify all possible anchor links across $G_s$ and $G_t$ by computing the alignment matrix $\mathbf{M}\in \mathbb{R}^{n_s \times n_t}$ without using any observed anchor links, where every entry $\mathbf{M}(i,j)$ represents the alignment score between node $i \in \mathcal{V}_s$ and node $j \in \mathcal{V}_t$.
\end{definition}

Generally, most existing work follows the consistency restraints shown in Fig. \ref{fig: alignment consistency}.
\begin{figure}[th]
    \setlength{\abovecaptionskip}{-0.0cm}   
    \setlength{\belowcaptionskip}{-0.3cm}  
    \centering
    \includegraphics[width=0.45\textwidth]{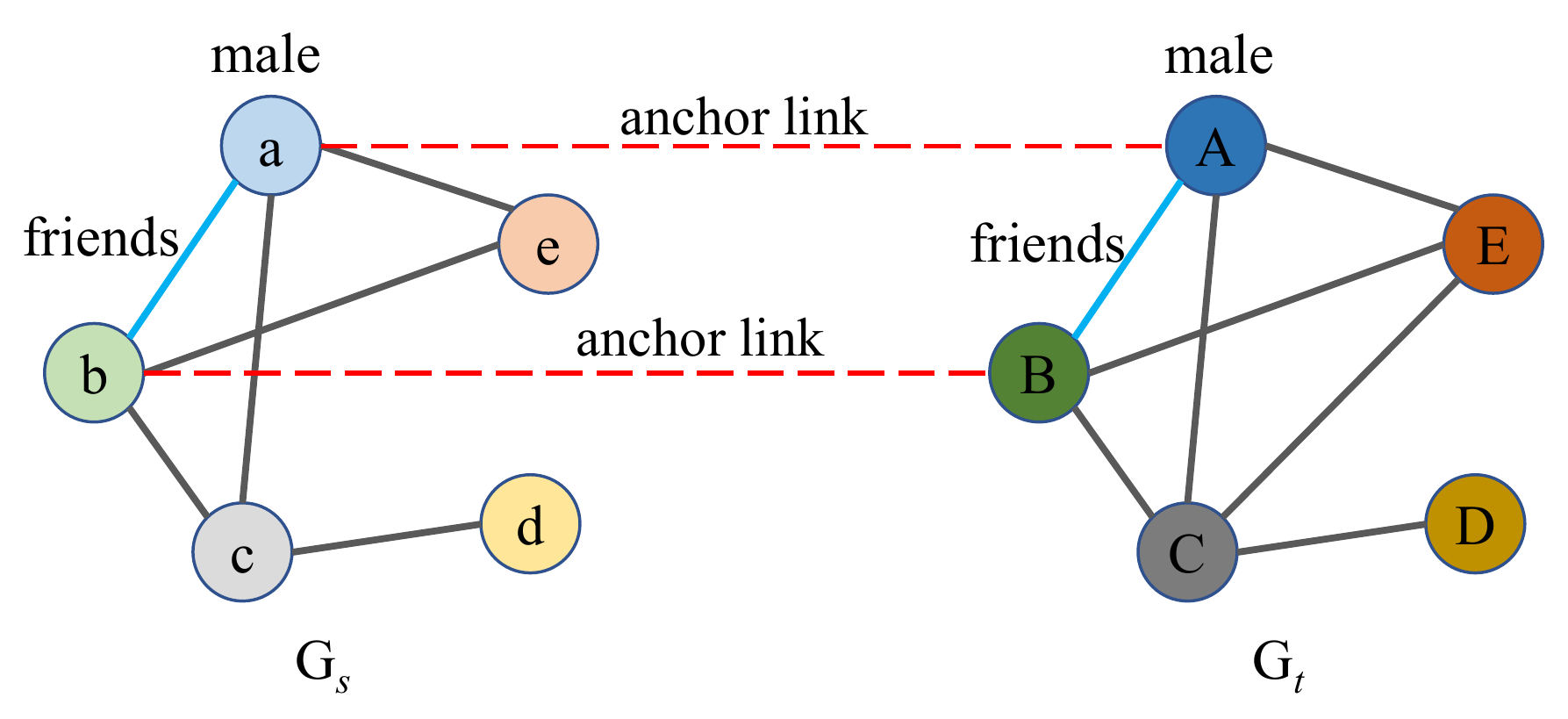}
    \caption{An illustrative example of alignment consistency.  Attribute Consistency: the gender attribute of node ``a'' and its corresponding node ``A'' are both ``male''; Topology Consistency: nodes ``a'' and ``b'' are connected in ${G_s}$, and their corresponding nodes ``A'' and ``B'' are also connected in ${G_t}$.}
    \label{fig: alignment consistency}
\end{figure}

\section{Methodology} \label{sec: methodology}
The framework of the proposed HTC is shown in Fig. \ref{fig: framework}. We introduce details of HTC in this section. 

\begin{figure*}[thb]
    \setlength{\abovecaptionskip}{-0.0cm}   
    \setlength{\belowcaptionskip}{-0.3cm}  
    \centering 
    \includegraphics[width=0.95\textwidth]{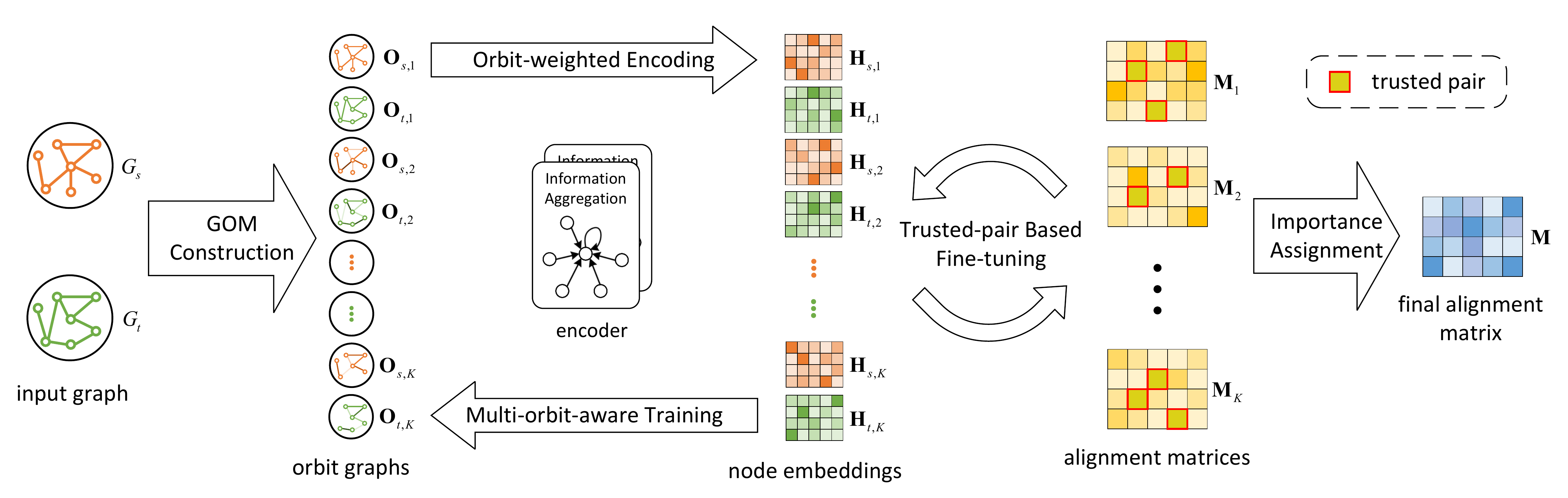}
    \caption{The framework of HTC.}
    \label{fig: framework}
\end{figure*}

\subsection{Higher-order Topological Consistency}
\begin{figure}[htbp]
    \setlength{\abovecaptionskip}{-0.0cm}   
    \setlength{\belowcaptionskip}{-0.3cm}  
    \centering   
    \includegraphics[width=0.48\textwidth]{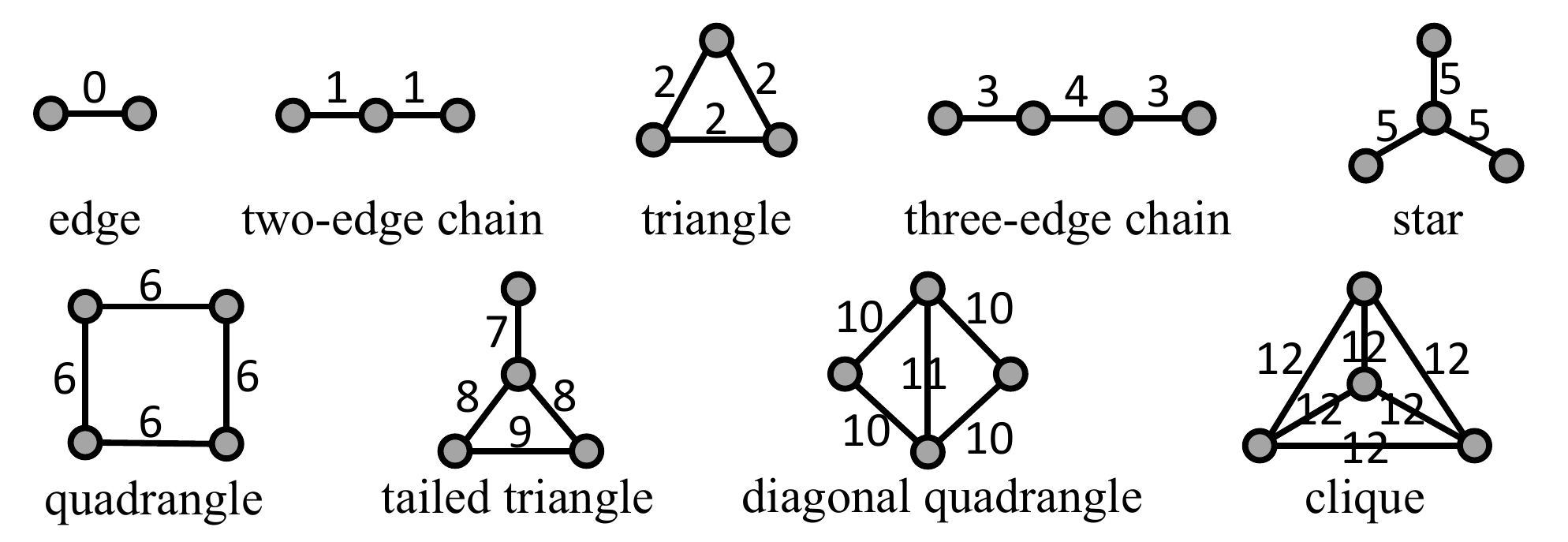}
    \caption{Induced graphlets with 2-4 nodes and automorphism orbits.}
    \label{fig: graphlets and orbits}
\end{figure}
\begin{figure}[htbp]
    \setlength{\abovecaptionskip}{0.0cm}   
    \setlength{\belowcaptionskip}{-0.2cm}  
    \centering   
    \includegraphics[width=0.48\textwidth]{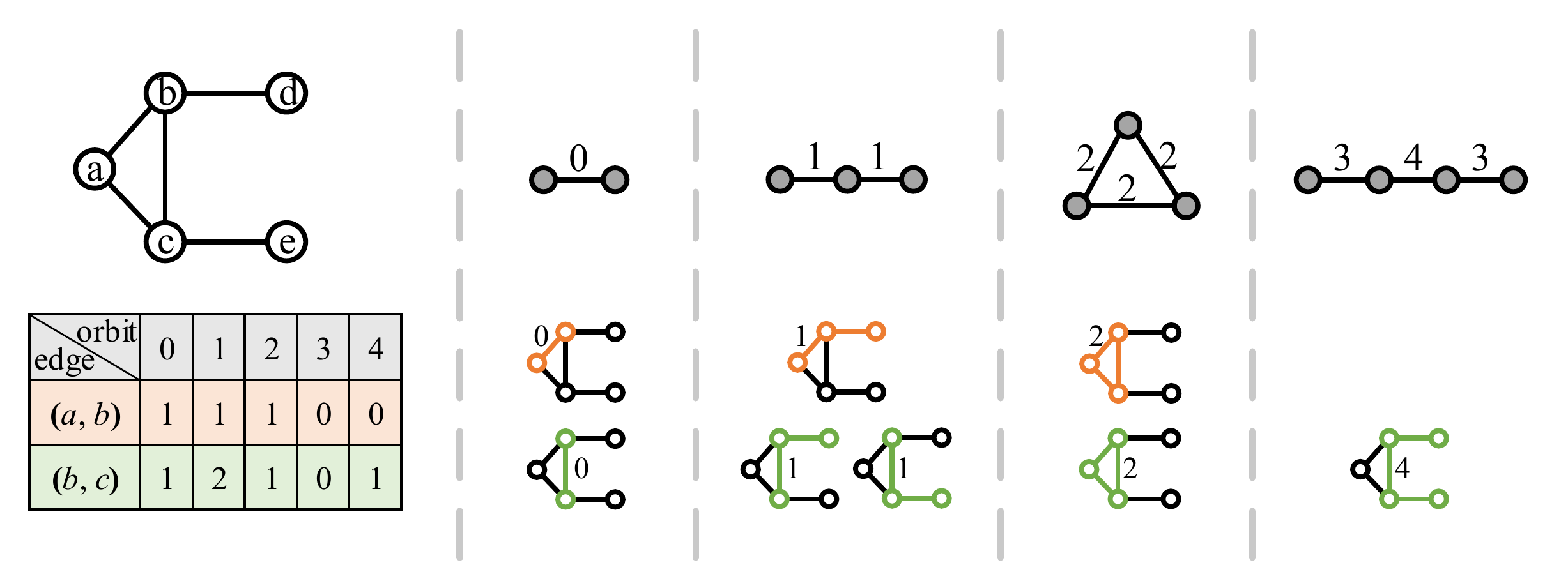}
    \caption{As an illustrative example, only the first five orbits are considered here. For node ${b}$, edge ${(a,b)}$ and ${(b,c)}$ are treated exactly the same if only consider two-node connectivity (orbit 0). However, taking other higher-order topological patterns into account, as listed in the left table, these two edges can be further distinguished since the frequencies at which they occur on orbit 1 and 4 are different, which indicates different roles that ${(a,b)}$ and ${(b,c)}$ play in the same subgraph.}
    \label{fig: illustration of orbit counting}
\end{figure}  
We first construct graphlet orbit matrix and define the higher-order topological consistency based on it.

\textit{Graphlet} is a term used to denote a connected network with a small number of nodes \cite{prvzulj2004modeling}. Any induced subgraph can be grouped into a certain type of graphlet if they are isomorphic. $G'=(\mathcal{V}',\mathcal{E}')$ is an \textit{induced subgraph} of $G=(\mathcal{V},\mathcal{E})$ if and only if $\mathcal{V}'\in \mathcal{V}$ and $\mathcal{E}'=\{(u,v)\in \mathcal{E} | u,v\in \mathcal{V}'\}$. Furthermore, the edges of each graphlet can be clustered into \textit{orbits} with respect to the graphlet automorphisms~\cite{solava2012graphlet}. For instance, there are 9 graphlets and 13 edge orbits for subgraphs with 2-4 nodes as shown in Fig. \ref{fig: graphlets and orbits}. Orbits define the ``roles" of edges within the graphlet. As an example, in the graphlet of a three-edge chain, one edge represents the bridge and the remaining two edges can be connected through the bridge; the edges of this graphlet thus form two different orbits (numbered 4 and 3, respectively). Moreover, edges around a node can be distinguished by comparing their occuring frequency on multiple orbits, as in the illustrative example shown in Fig. \ref{fig: illustration of orbit counting}.

We construct a set of \textit{graphlet orbit matrix (GOM)}, denoted as $\mathcal{O}=\{\mathbf{O}_0, \mathbf{O}_1, \cdots, \mathbf{O}_K\}$, where $K$ matrices correspond to $K$ orbits. Each GOM contains the frequency of edges on the corresponding orbit. Formally, given a graph $G$ with a vertex set of size $n$ and any two vertices $i$ and $j$ of it, the $k$th GOM $\mathbf{O}_k\in \mathbb{N}^{n \times n}$ is defined as follows:
\begin{equation}
\begin{split}
    \mathbf{O}_k(i,j)= &\text{ the number of times that edge } (i,j) \\
    &\text{ occurs on orbit } k.
\end{split}
\end{equation}

The greater the value of $\mathbf{O}_k(i,j)$ is, the stronger the connection between node $i$ and node $j$ is with respect to the orbit $k$. In addition to the above weighted definition, GOM can also be binary, in which $\mathbf{O}_k(i,j)$ equals to 1 as long as the edge $(i,j)$ is observed on orbit $k$ for at least one time, otherwise the value of $\mathbf{O}_k(i,j)$ is 0. In the binary case, the distinctions between edges are weakened. Referring to the example in Fig. \ref{fig: illustration of orbit counting}, the value of edge $(b,c)$ on orbit 1 would be 1, which is the same as that of edge $(a,b)$ if using binary setup. Hence, in this paper, we mainly consider the weighted form of GOM.  

Taking the value of each entry in GOM as the weight of the corresponding edge, GOMs exactly represent different levels of induced networks, as shown in Fig.~\ref{fig: orbit network}. Therefore, we can mathematically define higher-order topological consistency based on GOMs:

\begin{definition}
\textbf{High-order Topological Consistency.} 

Given two nodes $u, v \in \mathcal{V}_s$ and their corresponding nodes $u', v' \in \mathcal{V}_t$, they are regarded as satisfying $k$-order topological consistency, if their edges have identical weights in terms of the $k$-th GOM, \textit{i.e.,} $\mathbf{O}_{s,k}(u, v) = \mathbf{O}_{t,k}(u', v')$.
\end{definition}

\subsection{Orbit-weighted Encoding} \label{subsec: Encoding}
To integrate attribute consistency and higher-order topological consistency, we elegantly merge higher-order topological information into the aggregation process of GCN. Instead of treating all edges equally, we respect the fact that edges playing different roles in each high-order structure should be given different weights when passing messages.

\textbf{Weighted aggregation.} We naturally assign the orbit-defined weight to an edge for passing information between nodes. Formally, the feature of node $i$ on $l$th layer and $k$th orbit can be obtained by:  
\begin{equation}
    \mathbf{h}_{k,i}^l=f\left(\sum\nolimits_{j\in\mathcal{N}(i)}\mathbf{O}_k(i,j)\mathbf{h}_{k,j}^{l-1}\mathbf{W}^{l-1}\right),
    \label{eq: GCN}
\end{equation}
which can be performed efficiently in matrix formulation:
    $\mathbf{H}_{k}^l=f\left(\mathbf{O}_k\mathbf{H}_{k}^{l-1}\mathbf{W}^{l-1}\right)$, 
where $\mathbf{H}\in \mathbb{R}^{n \times d^{l-1}}$ is a feature matrix, $f$ is a nonlinear activation function, and $\mathbf{W}^{l-1}\in \mathbb{R}^{d^{l-1} \times d^l}$ is a trainable weight matrix for node representation encoding.

\textbf{Modified self-connection.} In Eq. (\ref{eq: GCN}), only information from neighbors is considered to generate a node's feature. To include the information from a node itself, we need to add self-connections. Note that the values of a weighted orbit matrix denote the frequency of edges occuring on a certain orbit, which may be far bigger than 1. In this case, using a typical identity matrix as the self-connection matrix is likely to result in the issue that the impact of self-connection is significantly weakened compared with its neighbors. Here, we define a more appropriate self-connection matrix with respect to orbit $k$ as follows:
\begin{equation}
    \mathbf{C}_k(i,i)=
    \begin{cases}
    1, & \text{if } \max\limits_{1 \le j \le n} \mathbf{O}_k(i,j) = 0;\\
    \max\limits_{1 \le j \le n} \mathbf{O}_k(i,j), & \text{otherwise}.
    \end{cases}
    \label{eq: self-loop}
\end{equation}
where $\max_{1 \le j \le n} \mathbf{O}_k(i,j) = 0$ means that node $i$ is not connected to any other node on orbit $k$. Intuitively, this matrix allows a node to either assign an importance to itself that is in line with that of its most important neighbor, or just consider its own information if there are no neighbors around it. Hence, the orbit matrix can be modified by $\mathbf{\tilde{O}}_k= \mathbf{O}_k + \mathbf{C}_k$.

\textbf{Normalized Laplacian.} According to the aggregation mechanism of GCN, nodes with large values in $\mathbf{\tilde{O}}_k$ will have large values in their feature representation, and vice versa. This may cause vanishing or exploding gradients \cite{kipf2016semi}. Hence, in practice, the procedure of normalization is required. The symmetric normalized Laplacian matrix on orbit $k$ can be obtained by $\mathbf{\tilde{L}}_k= \mathbf{\tilde{F}}_k^{-\frac{1}{2}}\mathbf{\tilde{O}}_k\mathbf{\tilde{F}}_k^{-\frac{1}{2}}$, where $\mathbf{\tilde{F}}_k(i,i)=\sum_{j=1}^n\mathbf{\tilde{O}}_k(i,j)$ is a diagonal frequency matrix.

\textbf{Forward encoding.} 
Say the GCN encoder contains $L$ hidden layers, the output embeddings of two graphs, denoted as $\mathbf{H}_{s,k}^L$ and $\mathbf{H}_{t,k}^L$ respectively, can be obtained through the forward encoding process: 
\begin{equation}
    \mathbf{H}_{s,k}^L=f^{L-1}\left(\mathbf{\tilde{L}}_{s,k}\cdots f^0\left(\mathbf{\tilde{L}}_{s,k}\mathbf{X}_s\mathbf{W}^0\right)\cdots \mathbf{W}^{L-1}\right),
    \label{eq: feedforward s}
\end{equation}
\begin{equation}
    \mathbf{H}_{t,k}^L=f^{L-1}\left(\mathbf{\tilde{L}}_{t,k}\cdots f^0\left(\mathbf{\tilde{L}}_{t,k}\mathbf{X}_t\mathbf{W}^0\right)\cdots \mathbf{W}^{L-1}\right).
    \label{eq: feedforward t}
\end{equation}
where the non-linear activation functions $\mathcal{F}=\{f^0, f^1, \cdots, f^{L-1}\}$ and the weight parameters $\mathcal{W}=\{\mathbf{W}^0, \mathbf{W}^1, \cdots, \mathbf{W}^{L-1}\}$ are \textbf{shared} between source and target graphs.

\textbf{Theoretical Analysis.} We mathematically prove that high-order structural consistency and attribute consistency will lead to identical node embeddings across two networks through our proposed encoding mechanism. First, we have the following lemma.

\begin{lemma}
For two nodes in the same graph $u$ and $u'\in \mathcal{V}$, if there exists a matching set $M=\{(v,v')|v\in \mathcal{N}(u), v'\in\mathcal{N}(u')\}$ and nodes of each matching pair have identical attributes, \textit{i.e.} $\mathbf{X}(v) = \mathbf{X}(v')$ and identical topological patterns in terms of orbit $k$, \textit{i.e.} $\mathbf{O}_{k}(u,v) = \mathbf{O}_{k}(u',v')$, then the embeddings of $u$ and $u'$ are identical after encoded by one GCN layer, \textit{i.e.} $\mathbf{H}_{k}(u)=\mathbf{H}_{k}(u')$.
    \label{lem: inner-similarity}
\end{lemma}

\begin{proof}
    On the one hand,
    \begin{equation*}
    \begin{split}
        \mathbf{H}_{k}(u) &=f^0\left(\mathbf{\tilde{L}}_{k}(u,:)\mathbf{X}\mathbf{W}^0\right)\\ &=f^0\left(\sum_{v\in \mathcal{N}(u)}\mathbf{\tilde{L}}_{k}(u,v)\mathbf{X}(v)\mathbf{W}^0\right);
    \end{split}
    \end{equation*}
    on the other hand,
    \begin{equation*}
    \begin{split}
        \mathbf{H}_{k}(u')&=f^0\left(\mathbf{\tilde{L}}_{k}(u',:)\mathbf{X}\mathbf{W}^0\right)\\
        &=f^0\left(\sum_{v'\in \mathcal{N}(u')}\mathbf{\tilde{L}}_{k}(u',v')\mathbf{X}(v')\mathbf{W}^0\right).
    \end{split}
    \end{equation*}
    Since $\mathbf{O}_{k}(u,v) = \mathbf{O}_{k}(u',v')$ for all $(v,v')$ pairs, we have $\mathbf{\tilde{L}}_{k}(u,v) = \mathbf{\tilde{L}}_{k}(u',v')$. Taking  $\mathbf{X}(v) = \mathbf{X}(v')$ into consideration, it is satisfied that $\sum_{v\in \mathcal{N}(u)}\mathbf{\tilde{L}}_{k}(u,v)\mathbf{X}(v) = \sum_{v'\in \mathcal{N}(u')}\mathbf{\tilde{L}}_{k}(u',v')\mathbf{X}(v')$. Hence, $\mathbf{H}_{k}(u)=\mathbf{H}_{k}(u')$.
\end{proof}

Actually, Lemma \ref{lem: inner-similarity} indicates that orbit-defined aggregation inherits the characteristics of the typical aggregation manner of GCN that attribute and topology similarity of nodes can be transformed into the similarity of node embeddings. To leverage the power of orbit-weighted aggregation for alignment tasks, we extend Lemma \ref{lem: inner-similarity} to the scenario of two networks and obtain the following proposition.

\begin{proposition}
    Suppose that two nodes $u \in \mathcal{V}_s$, $u' \in \mathcal{V}_t$ have an anchor link between them, and there exists a matching set $M=\{(v,v')|v\in \mathcal{N}_s(u), v'\in\mathcal{N}_t(u')\}$. All matching pairs $(v,v')$ satisfy the attribute consistency, \textit{i.e.} $\mathbf{X}_s(v) = \mathbf{X}_t(v')$ and the $k$-order topological consistency, \textit{i.e.} $\mathbf{O}_{s,k}(u,v) = \mathbf{O}_{t,k}(u',v')$, then the embeddings of $u$ and $u'$ after encoded by the shared GCN parameters are identical, \textit{i.e.} $\mathbf{H}_{s,k}(u)=\mathbf{H}_{t,k}(u')$.
    \label{pro: consistency to similarity}
\end{proposition}
\begin{proof}
    Similar to the proof of Lemma \ref{lem: inner-similarity}, we first expand the expression of embeddings of node $u$ and $u'$. For node $u$,
    \begin{equation*}
    \begin{split}
        \mathbf{H}_{s,k}(u)&=f_s^0\left(\mathbf{\tilde{L}}_{s,k}(u,:)\mathbf{X}_s\mathbf{W}_s^0\right)\\
        &=f_s^0\left(\sum_{v\in \mathcal{N}_s(u)}\mathbf{\tilde{L}}_{s,k}(u,v)\mathbf{X}_s(v)\mathbf{W}_s^0\right);    
    \end{split}    
    \end{equation*}
    for node $u'$,
    \begin{equation*}
    \begin{split}
        \mathbf{H}_{t,k}(u')&=f_t^0\left(\mathbf{\tilde{L}}_{t,k}(u',:)\mathbf{X}_t\mathbf{W}_t^0\right)\\
        &=f_t^0\left(\sum_{v'\in \mathcal{N}_t(u')}
        \mathbf{\tilde{L}}_{t,k}(u',v')\mathbf{X}_t(v')\mathbf{W}_t^0\right).
    \end{split}    
    \end{equation*}
    According to attribute consistency and topological consistency in terms of orbit $k$, we have $\mathbf{X}_s(v) = \mathbf{X}_t(v')$ and $\mathbf{\tilde{L}}_{s,k}(u,v) = \mathbf{\tilde{L}}_{t,k}(u',v')$.
    Besides, as the GCN encoder is shared between source and target graphs, $f_s^0 = f_t^0=f^0$ and $\mathbf{W}_s^0=\mathbf{W}_t^0=\mathbf{W}^0$ are satisfied. Hence, $\mathbf{H}_{s,k}(u) = \mathbf{H}_{t,k}(u')$. Similarly, for $L$-layer GCN, $\mathbf{H}_{s,k}^L(u)=\mathbf{H}_{t,k}^L(u')$ holds if corresponding $L$-hop neighborhood is consistent.
\end{proof}

The key factor in extending Lemma \ref{lem: inner-similarity} to Proposition \ref{pro: consistency to similarity} is the sharing of GCN encoder between two networks, which guarantees that nodes in two networks are embedded by the same mapping mechanism. Therefore, the consistency constraints of alignment can be transformed into the similarity of node embeddings and an extra step for aligning feature space can be omitted, which is claimed to be necessary in some previous embedding-based methods \cite{chen2020cone}. 

\subsection{Multi-orbit-aware Training} \label{subsec: Training}
As opposed to just using the trivial edge-defined topology information, multiple higher-order topological patterns give us more comprehensive information for aligning nodes.

\textbf{Orbit-reconstruction loss.} Without alignment labels, we adopt auto-encoder scheme to train the model parameters $\mathcal{W}$. We reconstruct each orbit Laplacian matrix using generated representations. The matrix of source graph on orbit $k$ is rebuilt by:  
\begin{equation}
    \mathbf{\hat{L}}_{s,k}=\mathbf{H}_{s,k}^L \cdot {\mathbf{H}_{s,k}^L}^\top.
    \label{eq: reconstruct}
\end{equation}
The reconstructed Laplacian matrix of target graph $\mathbf{\hat{L}}_{t,k}$ can also be obtained in a similar way. The orbit-reconstruction loss for orbit $k$ is further computed as follows:
\begin{equation}
    loss_k = {\left\|\mathbf{\tilde{L}}_{s,k}-\mathbf{\hat{L}}_{s,k}\right\|}_F + {\left\|\mathbf{\tilde{L}}_{t,k}-\mathbf{\hat{L}}_{t,k}\right\|}_F,
    \label{eq: loss k}
\end{equation}
where ${\|\cdot\|}_F$ denotes the Frobenius norm which is used to measure the difference between the original Laplacian matrix $\mathbf{\tilde{L}}_{\cdot,k}$ and the reconstructed one $\mathbf{\hat{L}}_{\cdot,k}$. By minimizing $loss_k$, the GCN encoder can be regarded as \textit{kth-orbit-aware} since the useful information required for reconstructing the $k$th orbit matrix is captured.



To be \textit{multi-orbit-aware}, all orbit-reconstruction loss items are added up to the overall objective:
\begin{equation}
    \Gamma = \sum_k loss_k.
    \label{eq: total loss}
\end{equation}

As $\Gamma$ is minimized by gradient descent optimizers, like Adam \cite{kingma2014adam}, the similarity of the embeddings generated by the encoder is able to reflect the corresponding topological and attribute consistency. The multi-orbit-aware process is summarized in Algorithm~\ref{alg: co_embedding}.

\begin{algorithm}[htbp]
\caption{\textsc{Multi-orbit-aware Embedding}}
\label{alg: co_embedding}
\begin{algorithmic}[1]

\REQUIRE Orbit graphs $\mathcal{O}=\{\mathbf{O}_0, \mathbf{O}_1, \cdots, \mathbf{O}_K\}$;
\ENSURE  Embeddings of source graph and target graph $\mathcal{H}_{s}=\{\mathbf{H}_{s,0}, \mathbf{H}_{s,1}, \cdots, \mathbf{H}_{s,K}\}$, $\mathcal{H}_{t}=\{\mathbf{H}_{t,0}, \mathbf{H}_{t,1}, \cdots, \mathbf{H}_{t,K}\}$.
\STATE Obtain modified Laplacian matrices $\tilde{\mathcal{L}}_{s}=\{\mathbf{\tilde{L}}_{s,0}, \mathbf{\tilde{L}}_{s,1}, \cdots, \mathbf{\tilde{L}}_{s,K}\}$, $\mathcal{\tilde{L}}_{t}=\{\mathbf{\tilde{L}}_{t,0}, \mathbf{\tilde{L}}_{t,1}, \cdots, \mathbf{\tilde{L}}_{t,K}\}$; 
\STATE Randomly initialize the encoding parameters $\mathcal{W}=\{\mathbf{W}^0, \mathbf{W}^1, \cdots, \mathbf{W}^{L-1}\}$;
\FOR{some epochs}
    \FOR{$k = 0 \rightarrow K$}
        \STATE Generate node embeddings $\mathbf{H}_{s,k}$, $\mathbf{H}_{t,k}$ of source and target graphs on orbit $k$ layer-by-layer using Eq. (\ref{eq: feedforward s})-(\ref{eq: feedforward t}); 
        \STATE Decode $\mathbf{H}_{s,k}$, $\mathbf{H}_{t,k}$ to obtain reconstructed Laplacian matrices $\mathbf{\hat{L}}_{s,k}$, $\mathbf{\hat{L}}_{t,k}$ respectively using Eq. (\ref{eq: reconstruct});
        \STATE Compute the reconstruction loss of source graph and target graph simultaneously using Eq. (\ref{eq: loss k});
    \ENDFOR
    \STATE Compute the total reconstruction loss $\Gamma$ of all patterns by using Eq. \ref{eq: total loss};
    \STATE Minimize $\Gamma$ to update $\mathcal{W}$ with the gradient descent optimizer;
\ENDFOR
\STATE Generate embeddings $\mathcal{H}_{s}$, $\mathcal{H}_{t}$;
\RETURN $\mathcal{H}_{s}$, $\mathcal{H}_{t}$

\end{algorithmic}
\end{algorithm}

\subsection{Trusted-pair based Fine-tuning} \label{subsec: Fine-tuning}
In this subsection, we use $\mathbf{h}_s$ and $\mathbf{h}_t$ to denote the embeddings of source graph and target graph output by the last hidden layer on any orbit, omitting the superscripts $L$ and subscripts $k$ for brevity. 

\textbf{Hubness problem.} By Lemma~\ref{lem: inner-similarity} and Proposition~\ref{pro: consistency to similarity}, both potential anchor nodes across different graphs and nodes with similar attributes and structure in the same graph will have embeddings as similar as possible. In high dimensional embedding space, these similar embeddings will be located closely. As a result, some nodes, called \textit{hubs}~\cite{dinu2014improving}, tend to be disproportionally nearest neighbors of many other nodes in the other graph. Direct nearest-neighbor rule favors these hubs, while at most, one of all the nearby nodes around a hub is a true anchor node. Thus, to diminish the hubness problem and identify nodes that are more isolated and more trustworthy to align with other nodes, we introduce the locally isolated similarity index (LISI).

\textbf{Locally isolated similarity index.} Due to the properties of translation invariance and scale invariance, we use Pearson correlation coefficient to measure the similarity between the embeddings of two nodes $\mathbf{h}_s \in \mathbf{H}_s$ and $\mathbf{h}_t\in \mathbf{H}_t$:
\begin{equation} 
    corr(\mathbf{h}_s, \mathbf{h}_t) = \frac{\sum\limits_i\left[\mathbf{h}_s\left(i\right)-\mathbf{\overline{h}}_{s}\right]\left[\mathbf{h}_{t}\left(i\right)-\mathbf{\overline{h}}_{t}\right]}{\sqrt{\sum\limits_i\left[\mathbf{h}_s\left(i\right)-\mathbf{\overline{h}}_{s}\right]^2\sum\limits_i\left[\mathbf{h}_{t}\left(i\right)-\mathbf{\overline{h}}_{t}\right]^2}},
    \label{eq: corr}
\end{equation}
where $\mathbf{\overline{h}}_s$, $\mathbf{\overline{h}}_t$ is the mean value of $\mathbf{h}_s$ and $\mathbf{h}_t$. 

We consider a bi-partite neighborhood graph, in which each node of a given graph is connected to its $m$ nearest neighbors in the other graph. Using $\mathcal{N}_t(\mathbf{h}_s)$ to denote the neighborhood of node embedding $\mathbf{h}_s$ on target graph space, the hubness degree of node $\mathbf{h}_s$ is measured by computing its mean similarity to its target neighborhood as follow:
\begin{equation} 
    D_t(\mathbf{h}_s)=\frac{\sum_{\mathbf{h}_t \in \mathcal{N}_t(\mathbf{h}_s)}corr(\mathbf{h}_s,\mathbf{h}_t)}{m}.
    \label{eq: mean sim}
\end{equation}
According to the definition, the larger $D_t(\mathbf{h}_s)$ is, the more likely $\mathbf{h}_s$ is to be a hub. Similarly, denoted by $D_s(\mathbf{h}_t)$, the mean similarity between $\mathbf{h}_t$ and its source neighborhood can be computed. They are used to calculate the similarity measure LISI:
\begin{equation}
    \mathrm{LISI}(\mathbf{h}_s,\mathbf{h}_t)=2corr(\mathbf{h}_s, \mathbf{h}_t) - D_t(\mathbf{h}_s) - D_s(\mathbf{h}_t).
    \label{eq: lisi}
\end{equation}

Here, a bigger value of LISI can be obtained when these two embeddings are similar enough to each other (the first item is large) while they are as less similar as possible to their neighbors in the other graph (the latter two items are small). In other words, LISI aims to identify two similar embeddings that are locally isolated rather than being hubs of dense areas.

\textbf{Trusted pairs.} The LISI values of node pairs across source graph and target graph are computed as the alignment matrix $\mathbf{M}$. If two nodes are each other's nearest neighbors with respect to LISI, they are regarded as a \textit{trusted pair}. Mathematically, a trusted pair $\langle\mathbf{h}_s,\mathbf{h}_t\rangle$ satisfies the following conditions:
\begin{equation}
    \begin{cases}
    \mathbf{h}_t = \arg\max\limits_{\mathbf{h}\in \mathbf{H}_t}\mathrm{LISI}(\mathbf{h}_s,\mathbf{h}),\\
    \mathbf{h}_s = \arg\max\limits_{\mathbf{h}\in \mathbf{H}_s}\mathrm{LISI}(\mathbf{h}_t,\mathbf{h}).
    \end{cases}
    \label{eq: mutual nn}
\end{equation}

\textbf{Aggregation coefficient fine-tuning.}  According to topology consistency, if two nodes are connected in a network, then their anchor nodes (if both exist) are likely to maintain connectivity in another network. In other words, if node $A$ and node $a$ form an anchor link across source graph and target graph, then potential anchor links are likely to exist across the neighborhood of node $A$ and that of node $a$. To find them, we have the following proposition.   

\begin{proposition}
By assigning bigger aggregation coefficients to trusted (or known) anchor nodes, more potential anchor nodes around them can be identified.
\end{proposition}

\begin{proof}
Without loss of generality, suppose that there exist four nodes: $u, v \in \mathcal{V}_s$ and $u', v' \in \mathcal{V}_t$, where $(v, v')$ is a pair of trusted (or known) anchor nodes, $(u,u')$ is a pair of potential anchor nodes that needed to be identified and $v \in \mathcal{N}_s(u), v' \in \mathcal{N}_t(u')$.
The first-layer embeddings of node $u$ and $u'$ can be expanded as:
\begin{equation*}
\begin{split}
    \mathbf{H}_{s,k}^1(u) &=  f^0\Bigg(\sum_{r\in \mathcal{N}_s\left(u\right)} \mathbf{\tilde{L}}_{s,k}\left(u,r\right) \mathbf{X}_s(r)\mathbf{W}^0\Bigg) \\
    &=  f^0\Bigg(\Bigg[\mathbf{\tilde{L}}_{s,k}(u,v) \mathbf{X}_s(v) \\
    & \qquad\quad+\sum_{r\in \mathcal{N}_s(u)\setminus v}\mathbf{\tilde{L}}_{s,k}(u,r)\mathbf{X}_s(r)\Bigg]\mathbf{W}^0\Bigg)\\
\end{split}
\end{equation*}
and 
\begin{equation*}
\begin{split}
    \mathbf{H}_{t,k}^1(u') 
    &= f^0\Bigg(\sum_{r'\in \mathcal{N}_t(u')} \mathbf{\tilde{L}}_{t,k}(u',r') \mathbf{X}_t(r')\mathbf{W}^0\Bigg) \\
    &= f^0\Bigg(\Bigg[\mathbf{\tilde{L}}_{t,k}(u',v') \mathbf{X}_t(v') \\
    &  \qquad\quad+\sum_{r'\in \mathcal{N}_t\left(u'\right)\setminus v'}\mathbf{\tilde{L}}_{t,k}(u',r')\mathbf{X}_t(r')\Bigg] \mathbf{W}^0\Bigg)\\
\end{split}
\end{equation*}
To increase the similarity between $\mathbf{H}_{s,k}^1(u)$ and $\mathbf{H}_{t,k}^1(u')$, we should make the terms in square brackets as close as possible. Note that $\mathbf{\tilde{L}}_{s,k}(u,\cdot)$ and $\mathbf{\tilde{L}}_{t,k}(u',\cdot)$ are scalars, which are generally regarded as aggregating weights, and $\mathbf{X}_s(\cdot)$ and $\mathbf{X}_t(\cdot)$ are vectors. In square brackets, the first item is a vector representing trusted/known anchor nodes, while the other item is a synthesized vector to aggregate the other nodes. As $(v, v')$ is a pair of trusted/known anchor nodes, $\mathbf{X}_s(v)=\mathbf{X}_t(v')$ holds due to attribute consistency. By enlarging the common vector components, namely the trusted-pair-related item, the synthesized vectors are tuned to be closer. In practice, since we do not know which pair of nodes around trusted pairs may have a potential anchor link, we can only tune $\mathbf{\tilde{L}}_{s,k}(:,v)$ and $\mathbf{\tilde{L}}_{t,k}(:,v')$ synchronously such that potential anchor links can be found and existing anchor links can be preserved. The process is similar for the following layers, which is omitted.
\end{proof}

In fact, assigning bigger aggregation coefficients to trusted anchor nodes means that their neighboring potential anchor nodes can receive more similar information. The more trusted nodes around potential anchor nodes, the easier they can be identified. Formally, two vectors $\mathbf{r}_s\in \mathbb{R}^{n_s}$ and $\mathbf{r}_t\in \mathbb{R}^{n_t}$ are maintained to update the reinforcement factors for nodes in source graph and target graph respectively. At the beginning, both of them are initialized to all-one vectors. If node $i \in \mathcal{V}_s$ and node $j \in \mathcal{V}_t$ are identified as trusted pairs, their corresponding reinforcement factors are updated as follows:
\begin{equation}
    \begin{cases}
    \mathbf{r}_s(i) = \beta \cdot \mathbf{r}_s(i),\\
    \mathbf{r}_t(j) = \beta \cdot \mathbf{r}_t(j),
    \end{cases}
    \label{eq: update reinforcement vector}
\end{equation}
where $\beta > 1$ is the reinforcement rate.

The aggregation scheme of GCN is modified as:
\begin{equation}
    \mathbf{H}_s^{l+1} = f^l\left(\mathbf{R}_s\mathbf{\tilde{L}_s}\mathbf{R}_s\mathbf{H}_s^l\mathbf{W}^l\right)
    \label{eq: new embedding}
\end{equation}
where $\mathbf{R}_s$ is a diagonal reinforcement matrix in which $\mathbf{R}_s(i,i)=\mathbf{r}_s(i)$. Likewise, the embeddings of target network are generated by introducing reinforcement matrix $\mathbf{R}_t$.

The fine-tuning process is iteratively conducted as shown in Fig. \ref{fig: framework}. Each loop contains four steps: (a) identify trusted pairs from the alignment matrix; (b) update aggregation coefficients of trusted pairs; (c) generate new node embeddings; (d) compute new alignment matrix. We also need to track the number of trusted pairs $\mathcal{T}$ in every loop. The objective is to maximize the number of trusted pairs identified based on each level of consistency. If the number stops growing, the loop is terminated (the loops for different orbits are independent of each other). The fine-tuning algorithm is summarized in Algorithm \ref{alg: fine-tuning}.
\begin{algorithm}[htbp]
\caption{\textsc{Trusted-pair Based Fine-tuning}}
\label{alg: fine-tuning}
\begin{algorithmic}[1]
{\REQUIRE Embeddings of source graph and target graph $\mathcal{H}_{s}$, $\mathcal{H}_{t}$, reinforcement rate $\beta$;
\ENSURE  Refined alignment matrices $\mathcal{M}=\{\mathbf{M}_0,\mathbf{M}_1 \cdots,\mathbf{M}_K\}$, the maximal number of trusted pairs $\mathcal{T}_{max}=\{T_{m0},T_{m1},\cdots,T_{mK}\}$.
\STATE Initialize reinforcement vectors $\mathbf{r}_s$, $\mathbf{r}_t$ with all-one vectors, number of trusted pairs $\mathcal{T}_{cur}=\{T_{c0},T_{c1},\cdots,T_{cK}\}$, $\mathcal{T}_{max}$ with all-zero vectors;
\FOR{$k = 0 \rightarrow K$}
    \REPEAT
        \STATE ${T}_{mk} \leftarrow {T}_{ck}$;
        \STATE Compute LISI as alignment matrix $\mathbf{M}_k$ using Eq. (\ref{eq: corr})-(\ref{eq: lisi});
        \STATE Identify and count trusted pairs ${T}_{ck}$ from $\mathbf{M}_k$ using Eq. (\ref{eq: mutual nn});
        \STATE Update $\mathbf{r}_s$, $\mathbf{r}_t$ using Eq. (\ref{eq: update reinforcement vector}); 
        \STATE Generate new embeddings $\mathbf{H}_{s,k}$, $\mathbf{H}_{t,k}$ using Eq. (\ref{eq: new embedding});
    \UNTIL ${T}_{ck} \le {T}_{mk}$
\ENDFOR
\RETURN $\mathcal{M}$, $\mathcal{T}_{max}$}
\end{algorithmic}
\end{algorithm}

\subsection{Posterior Importance Assignment}
After refinement procedure is done on each orbit, corresponding alignment matrices $\mathcal{M}=\{\mathbf{M}_0,\mathbf{M}_1 \cdots,\mathbf{M}_K\}$ are obtained. Each of them contains the alignment score of all node pairs, focusing on a specific order of topological consistency. To integrate those scores and produce the final alignment matrix, we consider an alignment-related index: the number of trusted pairs identified in different orders of topological consistency. The motivation is that the more trusted nodes that can be identified, the higher the corresponding embedding quality.

We denote the weight of  $\mathbf{M}_k$ by $\gamma_k$, which is defined as:
\begin{equation}
    \gamma_k = \frac{\mathcal{T}_k}{\sum_i\mathcal{T}_i},
    \label{eq: orbit importance}
\end{equation}
where $\mathcal{T}_i$ denotes the number of trusted pairs identified based on orbit $i$. Then the final alignment matrix is the weighted sum of each orbit-specific matrix: $\mathbf{M} = \sum_k \gamma_k \cdot \mathbf{M}_k$.

Given the final alignment matrix, for each node in source graph, the node with the highest score in target graph is regarded as the most possible anchor node. 

\subsection{Complexity Analysis} \label{subsec: complexity}
Without loss of generality, we denote the number of nodes by $n$, the number of edges by $e$ (non-zero entries of adjacency matrix), the maximal node degree by $D$, and the dimension of hidden features by $d$. 

\textbf{Time complexity.} There are mainly three steps to analyze:
\begin{itemize}
    \item Orbit matrix construction. In this work, we use orbits in graphlets with not more than 4 nodes. Thanks to the scalable Orca algorithm \cite{hovcevar2014combinatorial, hovcevar2016computation}, edge orbit counting can be done in $O(eD^2)$ for 4-node graphlets, which allows us to count orbits on graphs with millions of edges within a second. And constructing orbit-based Laplacian matrices takes $O(e)$.
    \item Multi-orbit-aware embedding. The propagation of GCN encoder takes $O(Ked^L)$  where the numbers of orbits $K$ and hidden layers $L$ are usually small constants.
    \item Trusted-pair based fine-tuning. Identifying trusted-pairs takes $O(Kn^2)$, and refining the aggregation coefficients takes $O(n)$.
\end{itemize}

Therefore, the first two steps take $O(e)$ and the third step takes $O(n^2)$. Note that $O(n^2)$ is dominated by the global alignment task itself ($\mathbf{M}\in \mathbb{R}^{n_s \times n_t}$), and most operations, such as trusted-pair identification, can be efficiently implemented through matrix manipulation. Also, considering the practical running time shown in Fig.~\ref{fig: runtime comparison}, the overall time cost of our approach is less than or comparable to that of most benchmark methods.

{\textbf{Space complexity.} We analyze space complexity from three aspects:
\begin{itemize}
    \item Orbit-relevant storage. Due to the increasing sparsity of higher-order orbits, there are at most $O(e)$ non-zero entries in each orbit matrix. Therefore, the worst-case space complexity for storing $K$ sparse orbit Laplacian matrices is $O(Ke)$.
    \item Embedding-relevant storage. It takes $O(Ld^2)$ to store trainable parameters in $L$ GCN layers, and the space required for storing the features of $n$ nodes in $L$ hidden layers is $O(Lnd)$.
    \item Alignment-relevant storage. In the fine-tuning stage, the LISI matrix is computed for identifying trusted pairs and can be released from memory once the identification process is done. Hence, it takes $O(n^2)$ to store the LISI matrix in the whole refinement process. As for the weighted integration stage, a cumulative approach can be adopted, which only requires maintaining a main matrix and iteratively adding each orbit-specified score matrix, and thus it also takes $O(n^2)$, the same as most alignment methods. 
\end{itemize}
\quad In summary, the overall space complexity is $O(Ke+Ld^2+Lnd+n^2)$. As the space complexity of a common GCN-based alignment approach would be $O(e+Ld^2+Lnd+n^2)$, our method generally just requires extra space for storing \textit{sparse} orbit matrices.
}

\section{Experimental Evaluations} \label{sec: experiment}
In this section, we comprehensively evaluate the superiority
of our unsupervised alignment framework on both real-world and synthetic datasets and compare it with state-of-the-art unsupervised and supervised methods. Other detailed evaluations are also reported which empirically demonstrate the effectiveness of the proposed method. 
\subsection{Experiment Setup}
{\textbf{Datasets.} 
To facilitate a fair comparison, in this work, we evaluate our approach by following the setup of the primary baselines \cite{zhang2016final, trung2020adaptive, trung2020comparative} and using the same datasets as they do, including three widely used real-world benchmark datasets and two synthetic datasets. For synthetic datasets, the source networks are collected from the real world, while the target networks are generated from source networks in the same way as \cite{trung2020adaptive} by randomly removing a certain percentage of the edges in original networks. Node identity is preserved, indicating the alignment groundtruth.}

$\bullet$ \textit{Allmovie \& Imdb}: each node is a movie and each edge between two movies signifies that they have at least one common actor. The former is constructed from Rotten Tomatoes website\footnote{https://www.kaggle.com/ayushkalla1/rotten-tomatoes-movie-database} and the latter from Imdb website\footnote{https://www.kaggle.com/jyoti1706/imdbmoviesdataset}. There are totally 5176 anchor links across these two networks~\cite{trung2020adaptive}.

$\bullet$ \textit{Douban Online \& Douban Offline}: a pair of Chinese social networks, whose nodes denote users and edges represent whether two users are contacts or friends. This pair of networks contains 1118 anchor links~\cite{zhang2016final, trung2020adaptive, trung2020comparative}.

{$\bullet$ \textit{Flickr \& Myspace}: Two subnetworks and their corresponding groundtruths are extracted by \cite{zhang2016final}. Each node represents a user, and part of the user's profile information is treated as attributes. There are 267 anchor links across these two networks~\cite{zhang2016final, trung2020comparative}}.

{$\bullet$ \textit{Econ}: this dataset comes from a economic model of Victoria State, Australia during the banking crisis in 1880 \cite{rossi2015network, trung2020adaptive}. Each node is an organisation (like banks, firms), and the edges denote the contractual relationships between them. Partial edges (from 10\% to 50\%) are randomly removed to construct the target network.}

{$\bullet$ \textit{BN}: this network represents a part of brain structure. The nodes are brain voxels, and each edge is a fiber tract that connects two voxels \cite{amunts2013bigbrain, trung2020adaptive}. Partial edges (from 10\% to 50\%) are randomly removed to construct the target network.}

The detailed statistics of aforementioned datasets are listed in Table. \ref{tab: statis}.   
\begin{table}[htbp]
\caption{\centering {{Statistical details of networks}}}
\begin{center}
\small
\begin{tabular}{lrrrr}
\toprule
Networks & \#Edges & \#Nodes & \#Attrs & Avg. Deg\\
\midrule
Allmovie & 124709 & 6011 & 14 & 41.4\\
Imdb & 119073 & 5713 & 14 & 41.7\\
Douban Online & 8164 & 3906 & 538 & 4.2 \\
Douban Offline & 1511 & 1118 & 538 & 2.7 \\
{Flickr} & {7333} & {6714} & {3} & {2.2} \\
{Myspace} & {11081} & {10733} & {3} & {2.1}\\
Econ & 7619 & 1258 & 20 & 12.1\\
BN & 9016 & 1781 & 20 & 10.1\\
\bottomrule
\end{tabular}
\label{tab: statis}
\end{center}
\end{table}

\textbf{Configurations for proposed framework.} If not stated otherwise, the hyperparameters are set as follows: number of GCN layers $L=2$, embedding dimension $d=200$, learning rate $\eta=0.01$, number of nearest neighbors $m=20$, aggregation reinforcement rate $\beta=1.1$. The sensitivity of key hyperparameters is studied in \ref{subsec: hyperparameter}.

\textbf{Baseline methods.} Other six representative methods are constructed for comparison, including methods that only use topology information ({IsoRank}~\cite{singh2008global}, {PALE}~\cite{heimann2018regal},  {CENALP}~\cite{du2019joint}) and three methods that use both topology information and node attributes ({FINAL}~\cite{zhang2016final}, {REGAL}~\cite{heimann2018regal},  {GAlign}~\cite{trung2020adaptive}). In particular, GAlign is the most related competitor, which is also an unsupervised network alignment method based on GCN. Since other general embedding techniques, like DeepWalk~\cite{perozzi2014deepwalk}/Node2Vec~\cite{grover2016node2vec}/LINE~\cite{tang2015line}, have been discussed and surpassed in papers of PALE/CENALP/REGAL, we compare with these alignment methods directly. According to the original settings of these baseline methods, some of them  have to introduce supervised information. Therefore, 10\% of ground truth is used to train PALE and CENALP and generate the prior alignment matrix for FINAL and IsoRank. 

\begin{table*}[ht]
\caption{\centering {{Numerical evaluations of alignment performance on three real-world datasets}}}
\small
\begin{center}
\begin{tabular}{lcccccccccccc}
\toprule
\multirow{2.5}{*}{Methods} & \multicolumn{4}{c}{Allmovie \& Imdb} & \multicolumn{4}{c}{Douban Online \& Offline} &  \multicolumn{4}{c}{Flickr \& Myspace} \\
\cmidrule(lr){2-5} \cmidrule(lr){6-9} \cmidrule(lr){10-13}
& {$p@1$} & $p@10$ & $\mathrm{MRR}$ & Time(s) & $p@1$ & $p@10$ & $\mathrm{MRR}$ & Time(s) & $p@1$ & $p@10$ & $\mathrm{MRR}$ & Time(s) \\
\midrule
HTC & \textbf{0.8436} & \underline{0.9051} & \textbf{0.8574} & {87.51} & \textbf{0.4651}  & \textbf{0.8005} & \textbf{0.5766} & \textbf{8.15} & \textbf{0.0150} & \underline{0.0487} & \textbf{0.0289} & \textbf{42.94}\\
GAlign  & \underline{0.8214} & 0.9003 & \underline{0.8496} & {92.48} & \underline{0.4526} & \underline{0.7800} & \underline{0.5632} & \underline{10.44} & \underline{0.0112} & {0.0412} & {0.0267} & \underline{48.29} \\
FINAL   & 0.7647 & \textbf{0.9609} & 0.8459 & \textbf{67.71} & 0.4383 & 0.7710 & 0.5539 & {183.57} & {0.0075} & {0.0412} & {0.0212} & {358.2}\\
PALE    & 0.6947 & 0.7159 & 0.7601 & 1810.28 & 0.0775 & 0.4479 & 0.1901 & 86.82 & 0.0000 & 0.0135 & 0.0079 & 166.11\\
CENALP  & 0.4866 & 0.8327 & 0.5693 & 43572.24 & 0.2572 & 0.4618 & 0.3537 & 8871.37 & 0.0075 & 0.0375 & 0.0244 & 21264.59 \\
IsoRank & 0.4653 & 0.6427 & 0.5271 & \underline{69.89} & 0.0903 & 0.2048 & 0.1299 & 19.29 & 0.0037 & 0.0337 & 0.0162 & 132.63\\
REGAL   & 0.0953 & 0.3869 & 0.1888 & 70.14 & 0.0528 & 0.2326 & 0.1120 & 20.47 & \underline{0.0112} & \textbf{0.0562} & \underline{0.0272} & 124.07\\
\midrule
Abs. Imp. & 0.0222 & - & 0.0078 & - & 0.0125 & 0.0205 & 0.0134 & 2.29 & 0.0038 & - & 0.0017 & 5.35 \\
Rel. Imp. & 2.70\% & - & 0.92\% & - & 2.76\% & 2.63\% & 2.38\% & 21.93\% & 33.93\% & - & 6.25\% & 11.08\% \\
\bottomrule
\end{tabular}
\label{tab: compare}
\end{center}
\end{table*}

\begin{figure*}[!ht]
\begin{subfigure}{0.33\textwidth}
\centering
\includegraphics[width=0.95\linewidth]{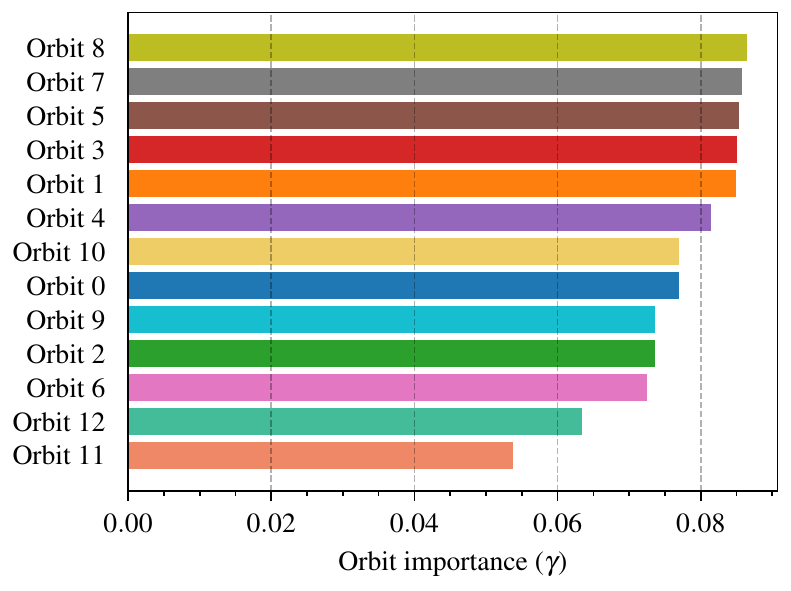} 
\caption{Allmovie \& Imdb}
\label{fig: mv}
\end{subfigure}
\begin{subfigure}{0.33\textwidth}
\centering
\includegraphics[width=0.95\linewidth]{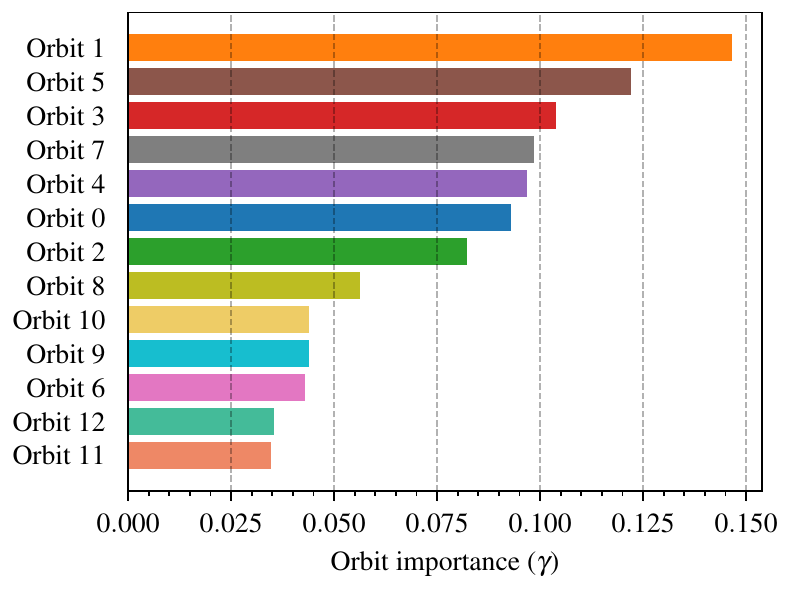}
\caption{Douban online \& offline}
\label{fig: db}
\end{subfigure}
\begin{subfigure}{0.33\textwidth}
\centering
\includegraphics[width=0.95\linewidth]{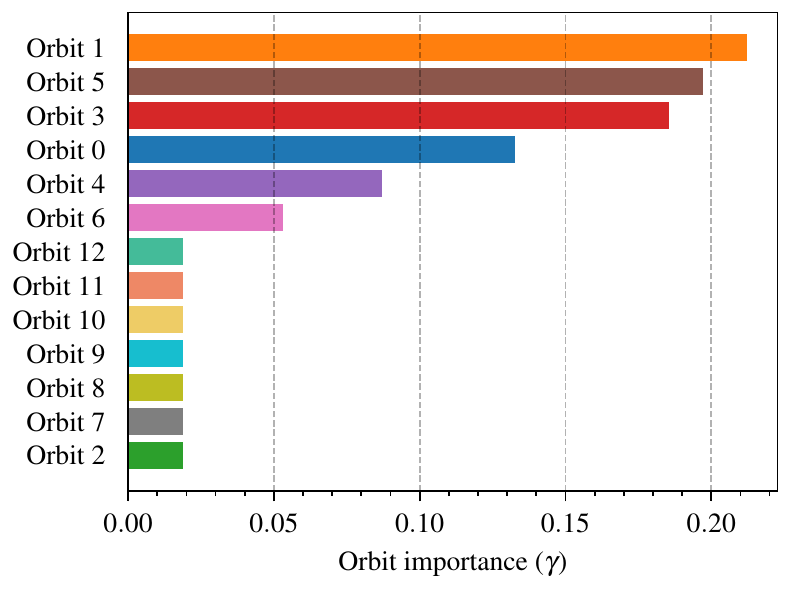}
\caption{Flickr \& Myspace}
\label{fig: fm}
\end{subfigure}
\caption{{The importance of orbits on three real-world datasets.}}
\label{fig: importance}
\end{figure*}
\textbf{Metrics.} Two numerical metrics are used to comprehensively evaluate and compare the effectiveness of different alignment approaches. The first metric is $precision@q$ (also known as $accuracy@q$ or $success@q$) \cite{zhang2016final}. It calculates the rate that a node's true anchor node can be found in its top-$q$ candidates. Mathematically, $precision@q$ is defined as:
\begin{equation}
    precision@q = \frac{\sum_{v_s^*\in \mathcal{V}_s}\sum_{v_t\in\mathcal{V}_t^q(v_s^*)}\mathds{1}\left(\left( v_s^*,v_t\right)\in\mathcal{L^*}\right)}{|\mathcal{L^*|}}
\end{equation}
where $\mathcal{L^*}$ denotes the set of ground-truth anchor pairs, $\mathcal{V}_t^q(v_s^*)$ denotes a subset of $\mathcal{V}_t$ which contains the target nodes that hold the top-$q$ highest alignment scores in the row $\mathbf{M}(v_s^*,\cdot)$, and $\mathds{1}(\cdot)$ is an indicator function.

Another metric, Mean Reciprocal Rank ($\mathrm{MRR}$) \cite{iofciu2011identifying}, also known as Mean Average Precision ($\mathrm{MAP}$) \cite{man2016predict}, evaluates the average performance of reciprocal rank. $\mathrm{MRR}$ is defined as follows:
\begin{equation}
    \mathrm{MRR} = \frac{1}{|\mathcal{L}^*|}\sum_i\frac{1}{r_i}
\end{equation}
where $r_i$ is the rank of $i$th true anchor node's alignment score in the corresponding row of $\mathbf{M}$. Note that for both $precision@q$ and $\mathrm{MRR}$, the higher the values are, the better the performance is. To mitigate randomness, all reported results are averaged over 20 runs.

{\textbf{Computing infrastructures.} All experiments are conducted on a compute node with 32 Intel Xeon Gold 6242 CPUs (2.80GHz) and 2 Tesla V100-SXM2 GPUs (32GB of memory each). As for the main softwares, all methods are implemented based on NetworkX-2.6.0 and PyTorch-1.11.0 (if required).}

\begin{figure*}[htbp]
\centering
\begin{minipage}[t]{0.48\textwidth}
\centering
\includegraphics[width=0.95\textwidth]{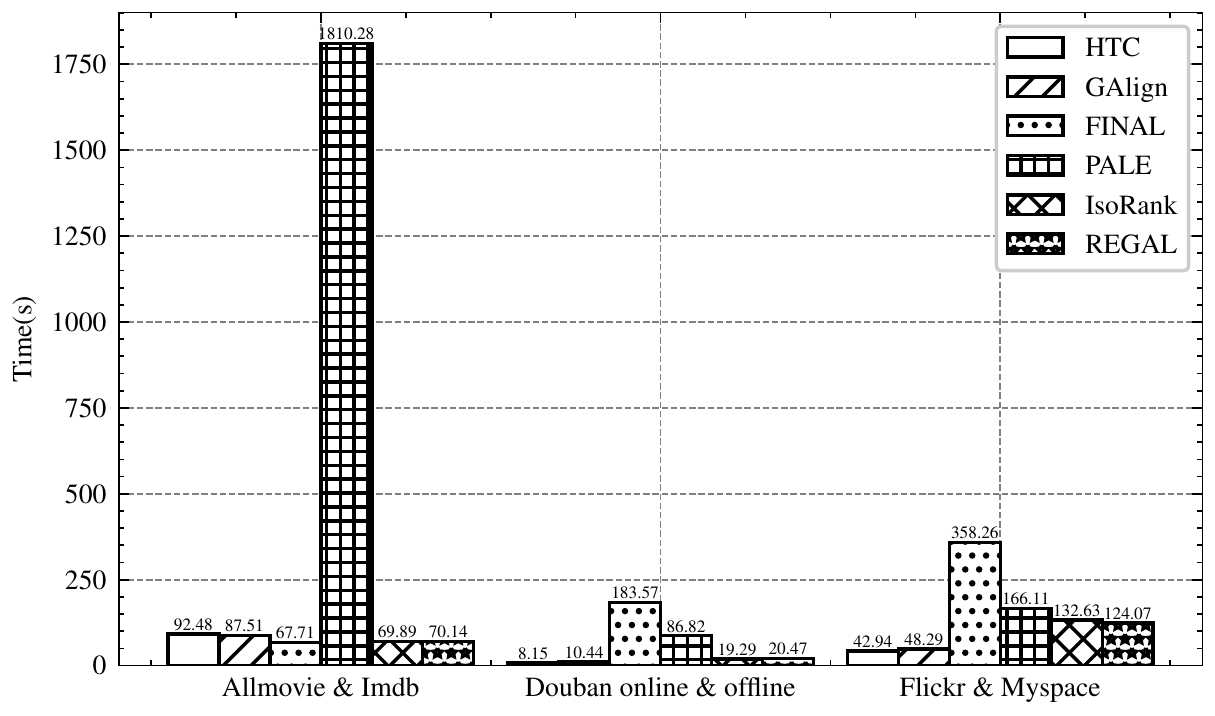}
\caption{{Runtime comparisons between HTC and baselines}}
\label{fig: runtime comparison}
\end{minipage}
\begin{minipage}[t]{0.48\textwidth}
\centering
\includegraphics[width=0.95\textwidth]{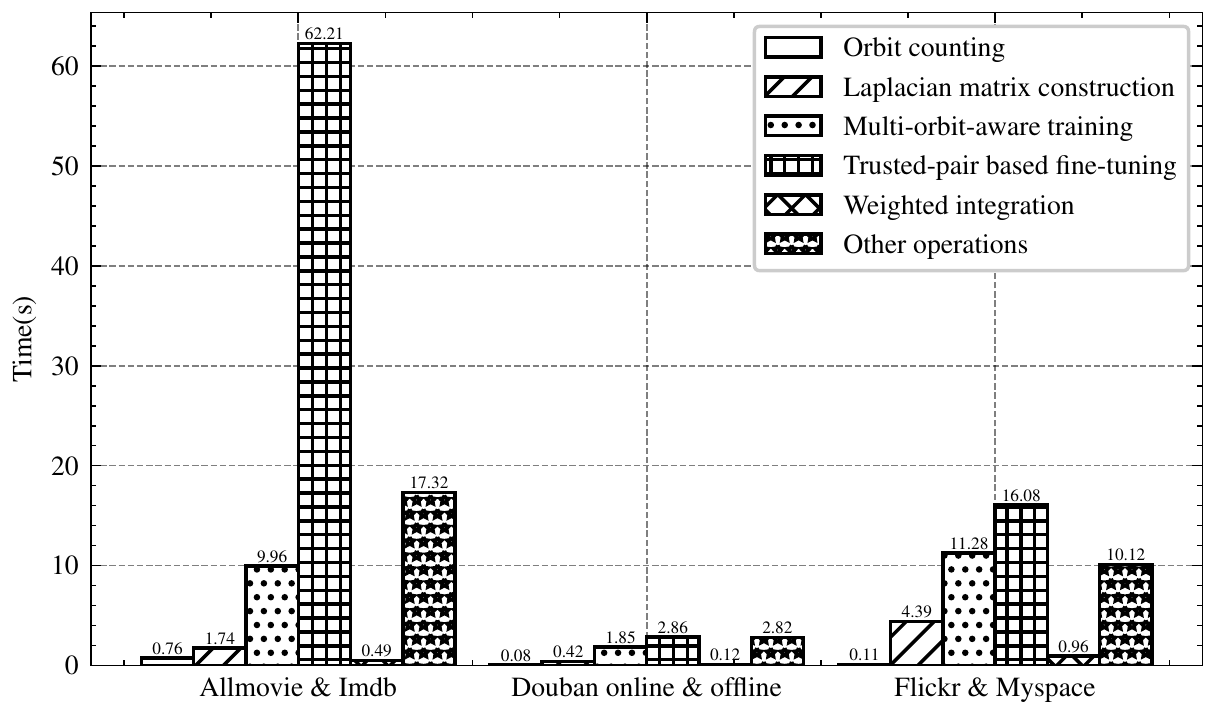}
\caption{{Runtime decomposition of HTC}}
\label{fig: runtime decomposition}
\end{minipage}
\end{figure*}

\subsection{Overall Effectiveness} \label{subsec: effectiveness}
The overall alignment performance is first investigated. We report the $precision@1$, $precision@10$, and $\mathrm{MRR}$ metrics for numerical comparison purpose, which are presented in Table \ref{tab: compare}. {As can be seen, the proposed HTC framework consistently outperforms all baseline models across three pairs of real-world datasets. Specifically, we achieve up to 33.93\%, 2.63\%, and 6.25\% relative improvements over the best results achieved by baselines in terms of $precision@1$, $precision@10$, and $\mathrm{MRR}$, respectively.}

Although 10\% of ground truth is provided for supervised methods, it seems inadequate to make them outstanding. Among them, FINAL is the best-performing supervised model, which achieves a competitive result in terms of $precision@10$ on Allmovie \& Imdb dataset. Even so, compared with it, our HTC is fully unsupervised, and is 10.31\% more accurate on Allmovie \& Imdb dataset in terms of $precision@1$, which is the most important metric.   

As for unsupervised competitors, GAlign is the most powerful one. Like our HTC, it does not focus solely on the original network topology. Instead, GAlign utilizes augmented networks to increase its adaptiveness to structural consistency violation, and thus it is the runner-up on three pairs of datasets in terms of several metrics. {In addition to steadily outperforming GAlign in terms of accuracy, our method likewise exhibits robustness to structural noise, which is further discussed in subsection~\ref{subsec: robustness}.}

{It is worth noting that all alignment methods perform poorly on Flickr \& Myspace datasets. In addition to very sparse network topology and very limited attribute information as listed in Table~\ref{tab: statis}, according to the insights provided in \cite{trung2020comparative}, another important factor that makes Flickr \& Myspace difficult to align is that their groundtruths rarely satisfy the common alignment consistency that neighbours in the source network have connections in the target network. As a result, all methods can hardly achieve satisfactory performance, especially for the $precision@1$ metric. However, our HTC is still the tallest among the short, and REGAL also shows relatively competitive results on this pair of datasets.}

{We also investigate the importance scores of different levels of orbit-weighted topology computed by Eq. \ref{eq: orbit importance} on three pairs of real-world datasets. In Fig. \ref{fig: importance}, we rank orbits by their $\gamma$ values, and the higher the position of the orbit in the plot, the higher its $\gamma$ value. Some interesting and insightful observations are as follows: (1) The overall distributions of orbit weights over three datasets are quite different. The weights on Allmovie \& Imdb datasets have the smallest variance than that of the other two pairs of datasets. This is consistent with network characteristics listed in Table \ref{tab: statis}, where Allmovie \& Imdb are far more dense, with an average degree greater than 40, and thus some higher-order patterns are more likely to exist and make contributions. In contrast, Flickr \& Myspace datasets are extremely sparse, which means they are less likely to form sufficient higher-order patterns, and thus only a few lower-order patterns account for most of the importance scores. The difference in the orbit importance distribution between these three datasets shows that our method can effectively and adaptively focus on more favourable topological consistency information according to specific network characteristics rather than maintaining fixed weights for each level. (2) Another key finding is that orbits 1, 3, and 5 rank in the top five in terms of the importance score in all three datasets, while orbit 0, the most basic edge pattern, is only included in the top five on Flickr \& Myspace datasets. This fact indicates that higher-order consistencies can actually make more contributions to the alignment tasks than the trivial one, which is the core factor that makes HTC work better than other baselines.}
\begin{figure*}[!ht]
\begin{subfigure}{0.49\textwidth}
\centering
\includegraphics[width=0.9\linewidth]{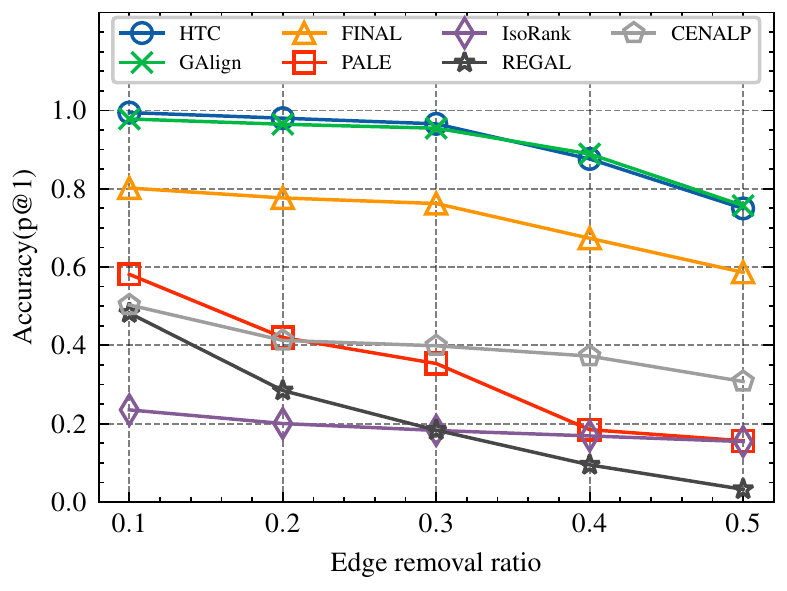} 
\caption{Econ}
\label{fig: econ}
\end{subfigure}
\begin{subfigure}{0.49\textwidth}
\centering
\includegraphics[width=0.9\linewidth]{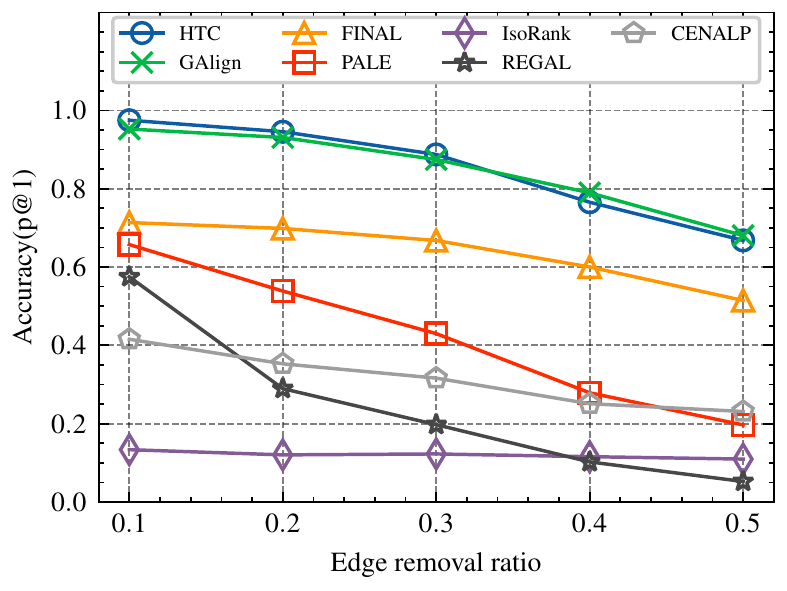}
\caption{BN}
\label{fig: bn}
\end{subfigure}
\caption{{Robustness test against topological noise on two synthetic datasets}}
\label{fig: robustness}
\end{figure*}

{\subsection{Efficiency Analysis}
To comprehensively evaluate the efficiency of our proposed HTC, the runtime of HTC on three pairs of real-world datasets is reported in Table~\ref{tab: compare}. Besides, to facilitate comparison with the baseline methods, we present runtime results in Fig.~\ref{fig: runtime comparison} (CENALP takes much more time and is therefore excluded). As illustrated, HTC takes the least time to achieve the best accuracy on two out of three datasets (Douban online \& offline and Flickr \& Myspace), and the time on another dataset (Allmovie \& Imdb) is also comparable to most benchmarks. Although multiple higher-order topological consistencies are considered in HTC, we do not introduce additional training-required encoding parameters as they are shared by all orbits, and thus will not result in much extra time cost. Also, the termination conditions of the fine-tuning processes on different orbits are independent of each other (refer to Algorithm~\ref{alg: fine-tuning}), which means redundant time consumption for extra running loops can be saved.}

{For an in-depth analysis, we divide the whole process of HTC into six parts, \textit{i.e.}, orbit counting, laplacian matrix construction, multi-orbit-aware training, trusted pair based fine-tuning, weighted integration, and other operations, and their corresponding time consumption is shown in Fig.~\ref{fig: runtime decomposition}. In particular, the time for other operations refers to the rest time consumption in addition to the other five parts of the total time, such as the time for calculating assessment metrics. On three datasets, orbit counting, Laplacian matrix construction, and weighted integration take relatively less time than the other processes. In addition to other operations, most of the time is spent on multi-orbit-aware training and trusted-pair based fine-tuning processes, which are the cores of our method.}

{In summary, taking both effectiveness and efficiency into consideration, our method shows its superiority over baselines.}


{\subsection{Robustness Test} \label{subsec: robustness}
To test the robustness of our proposed high-order topological consistency, in this subsection, we inject different levels of topological noise into synthetic datasets. Specifically, 10\% to 50\% of edges in original source networks are randomly removed to generate the corresponding target networks for Econ and BN datasets. The performance of our HTC is evaluated and compared with that of baselines, which is shown in Fig.~\ref{fig: robustness}. Generally, all methods suffer accuracy degradation as the noise level increases. Among those baselines, PALE and REGAL are more sensitive to structural noise than other baselines, while FINAL and CENALP are more robust. Besides, IsoRank performs less competitively under all different noise levels. In comparison, HTC consistently outperforms all baselines except GAlign by a remarkable margin. Although our method does not target noise adaptation as GAlign, it can still achieve better performance than GAlign when the noise ratio is smaller than 0.3 (specifically, when ratio equals 0.1, the $p@1$ metrics of HTC$|$GAlign on Econ and BN are 0.9944$|$0.9781 and 0.9748$|$0.9524, respectively), and the degree of performance degradation is generally comparable to GAlign, which is far smaller than the other baselines like PALE and REGAL and demonstrates the robustness of our HTC. For example, the performance degradations of HTC, GAlign, and PALE on Econ dataset are 0.2448, 0.2213, and 0.4252, respectively.}

{We owe the robustness of HTC to our multi-orbit-aware training mechanism based on two facts. First, not all higher order orbits are present for a given edge, \textit{e.g.}, the counts of orbit 3 and 4 for edge $(a,b)$ are both 0 as shown in Fig.~\ref{fig: illustration of orbit counting}, which is equivalent to removing the edge from the corresponding orbit. Second, the encoder is shared and co-trained by all different orders of topological patterns. As a result, HTC is trained to be robust to structural noise, \textit{i.e.}, edge removal/missing.}

\subsection{Ablation Test.}
To distinguish the contributions made by different components of HTC, we construct several variants and compare their performance on two real-world datasets (Because of the issues with Flickr \& Myspace mentioned in subsection~\ref{subsec: effectiveness}, we exclude them from the ablation test and the subsequent hyperparameter study). All of them are built based on a two-layer GCN and all related hyperparameters are consistent with HTC. 
\begin{itemize}
    \item \textit{HTC-L} (\underline{l}ow-order, w/o fine-tuning): only considers low-order topology, namely the simplest pattern (orbit 0) to learn node features and compute alignment matrix without a refinement process.
    \item \textit{HTC-H} (\underline{h}igh-order, w/o fine-tuning): utilizes multi-orbit-aware training techniques to learn higher-order topological consistency while removes the refinement step.
    \item \textit{HTC-LT} (\underline{l}ow-order, w/ fine-\underline{t}uning): learns node embeddings based on the simplest pattern, while the alignment matrix is computed after conducting trusted-pair based fine-tuning.
    \item \textit{HTC-DT} (\underline{d}iffusion, w/ fine-\underline{t}uning): replaces our GOMs with diffusion matrices of varying orders, which help to capture a larger neighborhood of structural information (best result is achieved with $k = 5$ and teleport probability $\alpha = 0.15$) \cite{klicpera2019diffusion}.
\end{itemize}

The results are reported in Table \ref{tab: ablation}. Compared with the other four variants, HTC (namely HTC-HT) performs far better in terms of both $p@1$ and $\mathrm{MRR}$ metrics. Specifically, HTC achieves up to a 49\% absolute improvement in $p@1$ over HTC-L, which ignores the proposed higher-order topological consistency and computes the alignment matrix directly without fine-tuning. This verifies the significant contributions we make to the original GCN in the context of unsupervised network alignment. Furthermore, HTC-H and HTC-LT outperform HTC-L by up to 38\% and 16\%, respectively, which indicates that more contributions are made by introducing higher-order topological consistency.

\begin{table}[!t]
\centering
\caption{Numerical results of ablation test}
\small
\begin{tabular}{lcccc}
\toprule
\multirow{2.5}{*}{Methods} & \multicolumn{2}{c}{Douban On/Off-line} & \multicolumn{2}{c}{Allmovie \& Imdb}\\
\cmidrule(lr){2-3} \cmidrule(lr){4-5} 
& {$p@1$} & $\mathrm{MRR}$ & $p@1$ & $\mathrm{MRR}$ \\
\midrule
HTC-L & 0.1978     & 0.3034     & 0.3515    & 0.4522   \\  
HTC-H & 0.4133     & 0.4941     & 0.7324    & 0.7418   \\     
HTC-LT & 0.2947    & 0.3988     & 0.5182    & 0.5892   \\     
HTC-DT & 0.0653     & 0.1508     & 0.1303    & 0.2214   \\ 
HTC(-HT) & \textbf{0.4651} & \textbf{0.5766} & \textbf{0.8436} & \textbf{0.8574} \\
\bottomrule
\end{tabular}
\label{tab: ablation}
\end{table}

\begin{figure*}[!ht]
\begin{subfigure}{0.49\textwidth}
\centering
\includegraphics[width=0.95\linewidth]{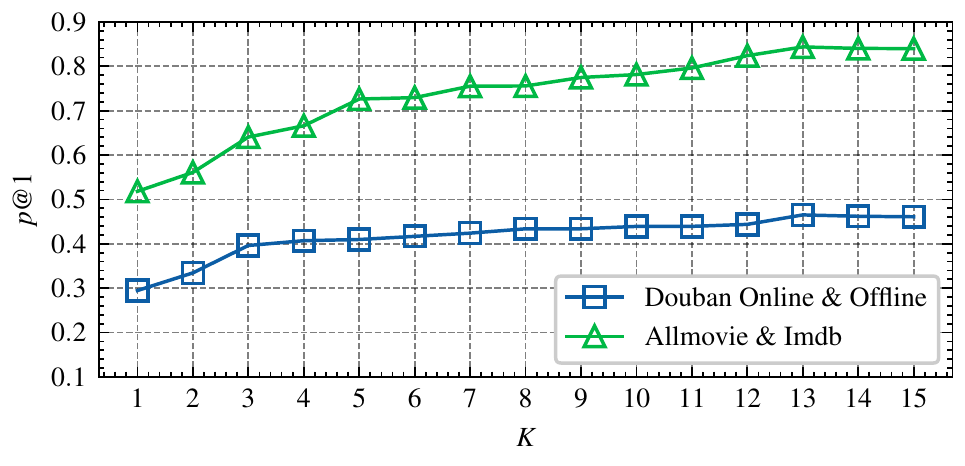} 
\caption{Number of orbits}
\label{fig: k}
\end{subfigure}
\begin{subfigure}{0.49\textwidth}
\centering
\includegraphics[width=0.95\linewidth]{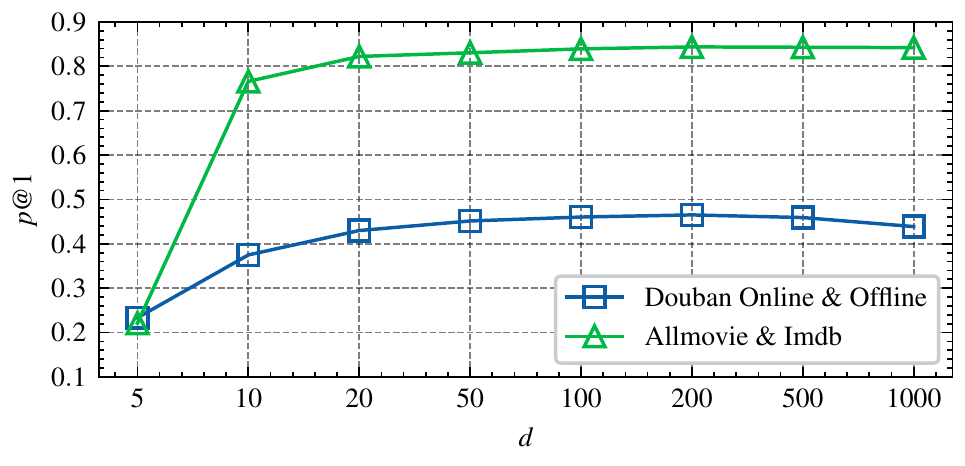}
\caption{Feature dimension of hidden layers}
\label{fig: d}
\end{subfigure}\\
\begin{subfigure}{0.49\textwidth}
\centering
\includegraphics[width=0.95\linewidth]{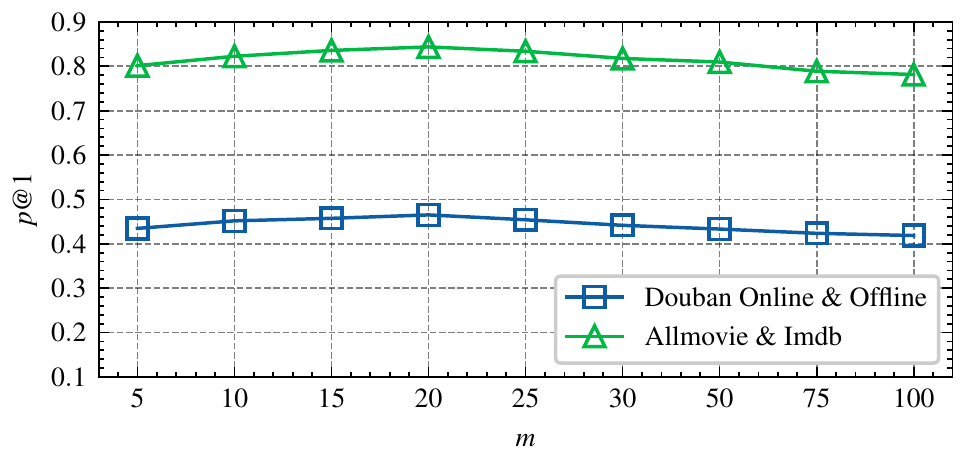}
\caption{Number of nearest neighbors}
\label{fig: m}
\end{subfigure}
\begin{subfigure}{0.49\textwidth}
\centering
\includegraphics[width=0.95\linewidth]{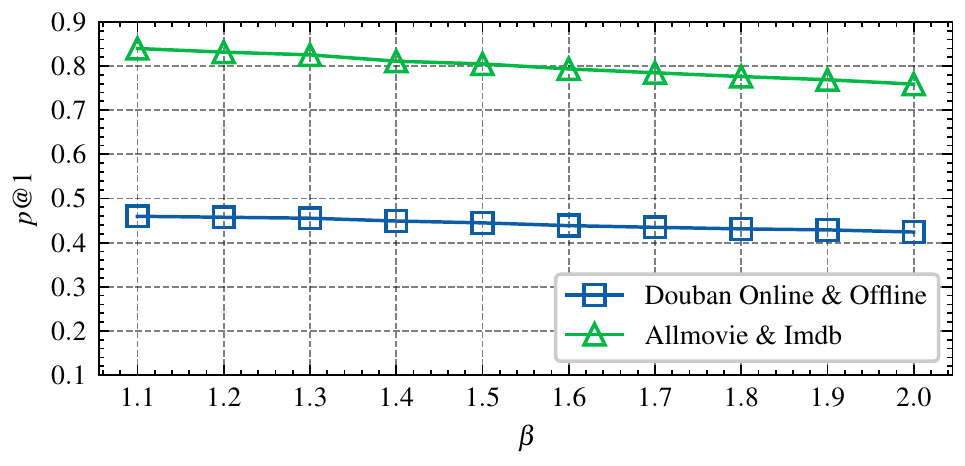}
\caption{Reinforcement rate}
\label{fig: beta}
\end{subfigure}
\caption{{The influence of key hyperparameters on $p@1$.}}
\label{fig: hyperparameter}
\end{figure*}
The diffusion matrix is demonstrated to be helpful for graph learning since multi-hop neighborhoods are considered instead of direct one-hop neighbors \cite{klicpera2019diffusion}. Thus, we take it as a representative technique that considers a larger range of neighboring topologies based on the trivial edge pattern for aggregation information (HTC-DT). However, it performs even worse than HTC-L in the unsupervised alignment task. The fact is, GOMs can generally provide more sparse but more accurate topological consistency information when the order increases, while diffusion matrices are quite the opposite. Densified matrices make it harder to identify good anchor links because structural specificity is reduced and node embeddings become more smooth. Therefore, simply considering a larger range of neighboring topologies may hardly provide favorable information, especially in the unsupervised alignment scenario.


\subsection{Hyperparameter Study} \label{subsec: hyperparameter}
The key hyperparameters of HTC are discussed in detail here.

(1) \textit{The number of orbits ($K$).} Orbits have two properties. First, the number of induced subgraphs on $N$ nodes increases exponentially with $N$ and so does the number of orbits defined on $N$ nodes. Second, the larger a graphlet is, the less frequently it occurs since such a graphlet requires more nodes and edges to form a specific local topology. As a result, the orbit matrix of a higher-order pattern is generally more sparse, which means less consistency information. The trend of incrementally using first $K$ orbits is displayed in Fig.~\ref{fig: k}. It can be found that increasing the number of orbits used at the beginning significantly improves the performance. The growth rate then decays, showing a smaller slope on the curve. The precision is no longer increasing and even slightly decreased when $K$ is larger than 13. Since considering too many orbits may bring little information gain (For instance, in Fig. \ref{fig: importance}, orbits 11, 12 account least for the model on both datasets) and cost more computation, in this work, we consider orbits in induced subgraphs with 2-4 nodes, \textit{i.e.} the first 13 orbits. 

(2) \textit{Structural hyperparameters of GCN ($L$; $d$).} It has been widely demonstrated that the best performance of a GCN model is achieved with 2 or 3 layers \cite{kipf2016semi, trung2020adaptive,lee2019graph}, and so is our model (we achieve best performance with $L=2$). As for the embedding dimension, it can be seen in Fig. \ref{fig: d} that the performance improves significantly at first and then slows down when $d$ grows from 5 to 200. Besides, the $p@1$ metric on Douban dataset decreases when $d>200$, which results from an over-fitting problem. It should be avoided by setting the dimension too large, which increases computational load and may even degrade performance. Hence, we set $d=200$. 

(3) \textit{The number of nearest neighbors ($m$).} The results of the controlled experiment are shown in Fig. \ref{fig: m}. Generally, both too small and too large values of $m$ lead to less satisfactory results, and the best performances on both datasets are achieved within the range [15, 25]. Hence, $m$ is set as 20 in this paper.

(4) \textit{Reinforcement rate ($\beta$).} It can be seen from Fig.~\ref{fig: beta} that the performance of HTC decreases as $\beta$ increases. This reveals that a small increase in the aggregation coefficient in each epoch can better detect more trusted pairs. Hence, we empirically set $\beta = 1.1$. 


\subsection{Visualization Analysis}
\begin{figure*}[!ht]
    \centering   
    \includegraphics[width=0.95\textwidth]{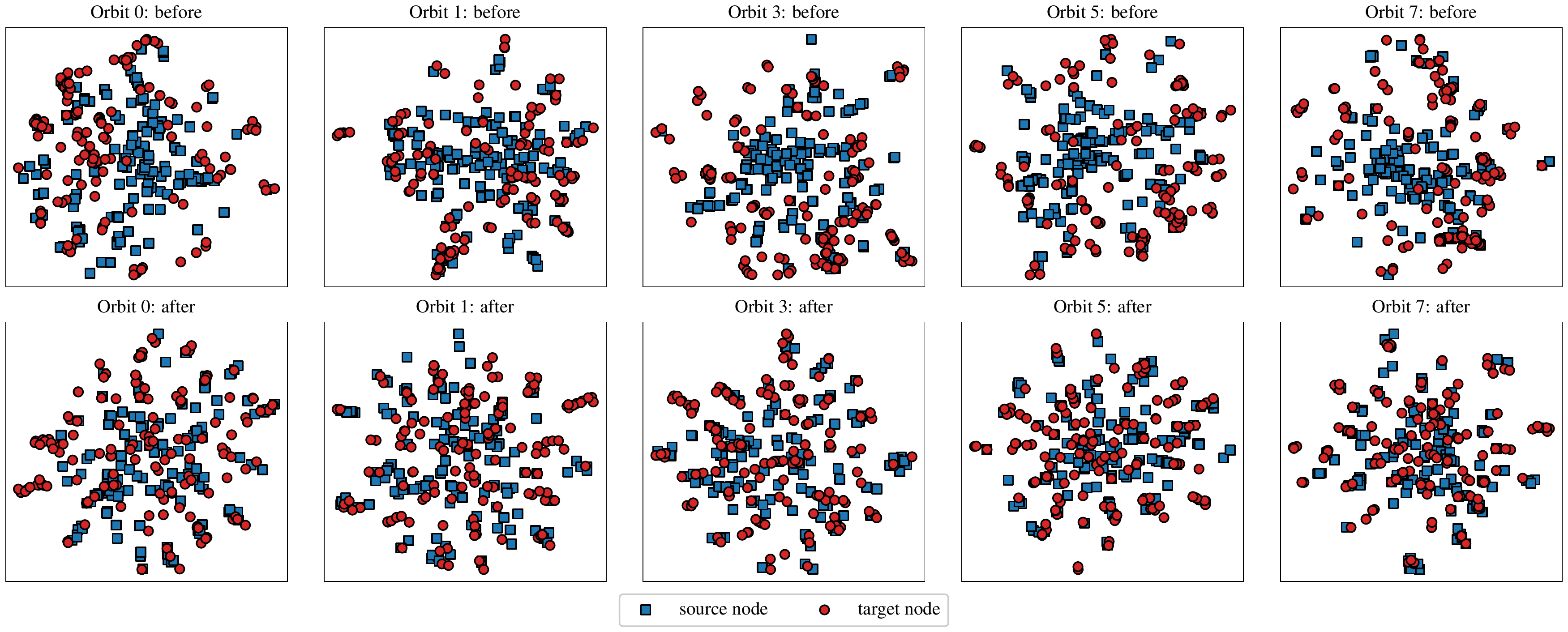}
    \setlength{\belowcaptionskip}{-0.2cm} 
    \caption{{The t-SNE visualization of node embeddings on Douban Online and Offline datasets}}
    \label{fig: tsne}
\end{figure*} 

We use t-SNE technique \cite{van2008visualizing} to visualize the embedding of ground truth anchor nodes before and after being aligned by HTC, as shown in Fig. \ref{fig: tsne}. We randomly sample 150 nodes from Douban Online dataset and their corresponding anchor nodes from Douban Offline dataset, which are mapped from high-dimensional feature space to 2-dimensional space. Due to space limitations, only five orbits are investigated and visualized here. It can be observed that the embeddings of two graphs before alignment have more distinct distributions (the upper five subgraphs). Many source graph embeddings (blue points) cannot find target graph embeddings (red points) that are close enough to them. While after alignment (the lower five subgraphs), the embeddings of source graph and target graph show a more similar distribution, and most node embeddings of source network can be covered by that of target network. This indicates that more consistency information between source graph and target graph is successfully captured by HTC.

\section{Conclusions} \label{sec: conclusion}
In this paper, we propose a fully unsupervised network alignment framework, named HTC, to develop topological consistency from lower order to higher order. We define higher-order topological consistency based on the introduced graphlet orbit matrix. By elegantly integrating higher-order consistency into the information aggregation process of GCN, the task of network alignment is converted into the task of node embedding similarity measurement. Instead of relying on any single level of topology, we train our encoder to be multi-orbit-aware in the unsupervised encoder-decoder paradigm. To further improve the alignment quality, we introduce a trusted-pair based fine-tuning mechanism. Accordingly, all orders of alignment consistency are integrated to obtain a comprehensive evaluation of node correspondence. {Extensive experimental results on benchmark datasets well demonstrate our superiority, such as effectiveness, efficiency, and robustness, over a wide variety of state-of-the-art baselines.} Our work verifies the significance of higher-order topological consistency for unsupervised alignment tasks. It also provides a general roadmap for improving the alignment performance of GCN-based models, as our method is built on the most common GCN-based alignment setting.



\bibliographystyle{ieeetr}
\bibliography{references}

\begin{thebibliography}{10}

\bibitem{dong2012link}
Y.~Dong, J.~Tang, S.~Wu, J.~Tian, N.~V. Chawla, J.~Rao, and H.~Cao, ``Link
  prediction and recommendation across heterogeneous social networks,'' in {\em
  2012 IEEE 12th International conference on data mining}, pp.~181--190, IEEE,
  2012.

\bibitem{xiang2018online}
Y.~Xiang, G.~Zhang, S.~Gu, and J.~Cai, ``Online multi-layer dictionary pair
  learning for visual classification,'' {\em Expert Systems with Applications},
  vol.~105, pp.~174--182, 2018.

\bibitem{li2014matching}
C.-Y. Li and S.-D. Lin, ``Matching users and items across domains to improve
  the recommendation quality,'' in {\em Proceedings of the 20th ACM SIGKDD
  international conference on Knowledge discovery and data mining},
  pp.~801--810, 2014.

\bibitem{liu2016aligning}
L.~Liu, W.~K. Cheung, X.~Li, and L.~Liao, ``Aligning users across social
  networks using network embedding.,'' in {\em Ijcai}, pp.~1774--1780, 2016.

\bibitem{yin2016adapting}
H.~Yin, X.~Zhou, B.~Cui, H.~Wang, K.~Zheng, and Q.~V.~H. Nguyen, ``Adapting to
  user interest drift for poi recommendation,'' {\em IEEE Transactions on
  Knowledge and Data Engineering}, vol.~28, no.~10, pp.~2566--2581, 2016.

\bibitem{nassar2018low}
H.~Nassar, N.~Veldt, S.~Mohammadi, A.~Grama, and D.~F. Gleich, ``Low rank
  spectral network alignment,'' in {\em Proceedings of the 2018 World Wide Web
  Conference}, pp.~619--628, 2018.

\bibitem{zhan2015influence}
Q.~Zhan, J.~Zhang, S.~Wang, S.~Y. Philip, and J.~Xie, ``Influence maximization
  across partially aligned heterogenous social networks,'' in {\em Pacific-Asia
  Conference on Knowledge Discovery and Data Mining}, pp.~58--69, Springer,
  2015.

\bibitem{Zhang2019}
J.~Zhang and P.~S. Yu, {\em Unsupervised Network Alignment}, pp.~165--202.
\newblock Cham: Springer International Publishing, 2019.

\bibitem{mu2016user}
X.~Mu, F.~Zhu, E.-P. Lim, J.~Xiao, J.~Wang, and Z.-H. Zhou, ``User identity
  linkage by latent user space modelling,'' in {\em Proceedings of the 22nd ACM
  SIGKDD International Conference on Knowledge Discovery and Data Mining},
  pp.~1775--1784, 2016.

\bibitem{trung2020adaptive}
H.~T. Trung, T.~Van~Vinh, N.~T. Tam, H.~Yin, M.~Weidlich, and N.~Q.~V. Hung,
  ``Adaptive network alignment with unsupervised and multi-order convolutional
  networks,'' in {\em 2020 IEEE 36th International Conference on Data
  Engineering (ICDE)}, pp.~85--96, IEEE, 2020.

\bibitem{zhu2020huna}
D.~Zhu, Y.~Sun, H.~Du, N.~Cao, T.~Baker, and G.~Srivastava, ``Huna: A method of
  hierarchical unsupervised network alignment for iot,'' {\em IEEE Internet of
  Things Journal}, vol.~8, no.~5, pp.~3201--3210, 2020.

\bibitem{zhou2020unsupervised}
Y.~Zhou, J.~Ren, R.~Jin, Z.~Zhang, D.~Dou, and D.~Yan, ``Unsupervised multiple
  network alignment with multinominal gan and variational inference,'' in {\em
  2020 IEEE International Conference on Big Data (Big Data)}, pp.~868--877,
  IEEE, 2020.

\bibitem{trung2020comparative}
H.~T. Trung, N.~T. Toan, T.~Van~Vinh, H.~T. Dat, D.~C. Thang, N.~Q.~V. Hung,
  and A.~Sattar, ``A comparative study on network alignment techniques,'' {\em
  Expert Systems with Applications}, vol.~140, p.~112883, 2020.

\bibitem{singh2008global}
R.~Singh, J.~Xu, and B.~Berger, ``Global alignment of multiple protein
  interaction networks with application to functional orthology detection,''
  {\em Proceedings of the National Academy of Sciences}, vol.~105, no.~35,
  pp.~12763--12768, 2008.

\bibitem{labitzke2011your}
S.~Labitzke, I.~Taranu, and H.~Hartenstein, ``What your friends tell others
  about you: Low cost linkability of social network profiles,'' in {\em Proc.
  5th International ACM Workshop on Social Network Mining and Analysis, San
  Diego, CA, USA}, pp.~1065--1070, 2011.

\bibitem{liu2013s}
J.~Liu, F.~Zhang, X.~Song, Y.-I. Song, C.-Y. Lin, and H.-W. Hon, ``What's in a
  name? an unsupervised approach to link users across communities,'' in {\em
  Proceedings of the sixth ACM international conference on Web search and data
  mining}, pp.~495--504, 2013.

\bibitem{zhang2016final}
S.~Zhang and H.~Tong, ``Final: Fast attributed network alignment,'' in {\em
  Proceedings of the 22nd ACM SIGKDD International Conference on Knowledge
  Discovery and Data Mining}, pp.~1345--1354, 2016.

\bibitem{heimann2018regal}
M.~Heimann, H.~Shen, T.~Safavi, and D.~Koutra, ``Regal: Representation
  learning-based graph alignment,'' in {\em Proceedings of the 27th ACM
  international conference on information and knowledge management},
  pp.~117--126, 2018.

\bibitem{man2016predict}
T.~Man, H.~Shen, S.~Liu, X.~Jin, and X.~Cheng, ``Predict anchor links across
  social networks via an embedding approach.,'' in {\em Ijcai}, vol.~16,
  pp.~1823--1829, 2016.

\bibitem{zhou2018deeplink}
F.~Zhou, L.~Liu, K.~Zhang, G.~Trajcevski, J.~Wu, and T.~Zhong, ``Deeplink: A
  deep learning approach for user identity linkage,'' in {\em IEEE INFOCOM
  2018-IEEE Conference on Computer Communications}, pp.~1313--1321, IEEE, 2018.

\bibitem{milo2002network}
R.~Milo, S.~Shen-Orr, S.~Itzkovitz, N.~Kashtan, D.~Chklovskii, and U.~Alon,
  ``Network motifs: simple building blocks of complex networks,'' {\em
  Science}, vol.~298, no.~5594, pp.~824--827, 2002.

\bibitem{prvzulj2004modeling}
N.~Pr{\v{z}}ulj, D.~G. Corneil, and I.~Jurisica, ``Modeling interactome:
  scale-free or geometric?,'' {\em Bioinformatics}, vol.~20, no.~18,
  pp.~3508--3515, 2004.

\bibitem{kipf2016variational}
T.~N. Kipf and M.~Welling, ``Variational graph auto-encoders,'' {\em arXiv
  preprint arXiv:1611.07308}, 2016.

\bibitem{du2019joint}
X.~Du, J.~Yan, and H.~Zha, ``Joint link prediction and network alignment via
  cross-graph embedding.,'' in {\em IJCAI}, pp.~2251--2257, 2019.

\bibitem{liang2021unsupervised}
Z.~Liang, Y.~Rong, C.~Li, Y.~Zhang, Y.~Huang, T.~Xu, X.~Ding, and J.~Huang,
  ``Unsupervised large-scale social network alignment via cross network
  embedding,'' in {\em Proceedings of the 30th ACM International Conference on
  Information \& Knowledge Management}, pp.~1008--1017, 2021.

\bibitem{mohammadi2016triangular}
S.~Mohammadi, D.~F. Gleich, T.~G. Kolda, and A.~Grama, ``Triangular alignment
  (tame): A tensor-based approach for higher-order network alignment,'' {\em
  IEEE/ACM transactions on computational biology and bioinformatics}, vol.~14,
  no.~6, pp.~1446--1458, 2016.

\bibitem{klicpera2019diffusion}
J.~Klicpera, S.~Wei{\ss}enberger, and S.~G{\"u}nnemann, ``Diffusion improves
  graph learning,'' {\em Advances in Neural Information Processing Systems},
  vol.~32, pp.~13354--13366, 2019.

\bibitem{monti2018motifnet}
F.~Monti, K.~Otness, and M.~M. Bronstein, ``Motifnet: a motif-based graph
  convolutional network for directed graphs,'' in {\em 2018 IEEE Data Science
  Workshop (DSW)}, pp.~225--228, IEEE, 2018.

\bibitem{lee2019graph}
J.~B. Lee, R.~A. Rossi, X.~Kong, S.~Kim, E.~Koh, and A.~Rao, ``Graph
  convolutional networks with motif-based attention,'' in {\em Proceedings of
  the 28th ACM International Conference on Information and Knowledge
  Management}, pp.~499--508, 2019.

\bibitem{sankar2019meta}
A.~Sankar, X.~Zhang, and K.~C.-C. Chang, ``Meta-gnn: Metagraph neural network
  for semi-supervised learning in attributed heterogeneous information
  networks,'' in {\em Proceedings of the 2019 IEEE/ACM International Conference
  on Advances in Social Networks Analysis and Mining}, pp.~137--144, 2019.

\bibitem{prvzulj2007biological}
N.~Pr{\v{z}}ulj, ``Biological network comparison using graphlet degree
  distribution,'' {\em Bioinformatics}, vol.~23, no.~2, pp.~e177--e183, 2007.

\bibitem{solava2012graphlet}
R.~W. Solava, R.~P. Michaels, and T.~Milenkovi{\'c}, ``Graphlet-based edge
  clustering reveals pathogen-interacting proteins,'' {\em Bioinformatics},
  vol.~28, no.~18, pp.~i480--i486, 2012.

\bibitem{milenkovic2010optimal}
T.~Milenkovi{\'c}, W.~L. Ng, W.~Hayes, and N.~Pr{\v{z}}ulj, ``Optimal network
  alignment with graphlet degree vectors,'' {\em Cancer informatics}, vol.~9,
  pp.~CIN--S4744, 2010.

\bibitem{crawford2015great}
J.~Crawford and T.~Milenkovi{\'c}, ``Great: graphlet edge-based network
  alignment,'' in {\em 2015 IEEE International Conference on Bioinformatics and
  Biomedicine (BIBM)}, pp.~220--227, IEEE, 2015.

\bibitem{almulhim2019network}
A.~Almulhim, V.~S. Dave, and M.~A. Hasan, ``Network alignment using graphlet
  signature and high order proximity,'' in {\em International Conference on
  Machine Learning, Optimization, and Data Science}, pp.~130--142, Springer,
  2019.

\bibitem{grover2016node2vec}
A.~Grover and J.~Leskovec, ``node2vec: Scalable feature learning for
  networks,'' in {\em Proceedings of the 22nd ACM SIGKDD international
  conference on Knowledge discovery and data mining}, pp.~855--864, 2016.

\bibitem{perozzi2014deepwalk}
B.~Perozzi, R.~Al-Rfou, and S.~Skiena, ``Deepwalk: Online learning of social
  representations,'' in {\em Proceedings of the 20th ACM SIGKDD international
  conference on Knowledge discovery and data mining}, pp.~701--710, 2014.

\bibitem{tang2015line}
J.~Tang, M.~Qu, M.~Wang, M.~Zhang, J.~Yan, and Q.~Mei, ``Line: Large-scale
  information network embedding,'' in {\em Proceedings of the 24th
  international conference on world wide web}, pp.~1067--1077, 2015.

\bibitem{kipf2016semi}
T.~N. Kipf and M.~Welling, ``Semi-supervised classification with graph
  convolutional networks,'' {\em arXiv preprint arXiv:1609.02907}, 2016.

\bibitem{nguyen2018graph}
T.~Nguyen and R.~Grishman, ``Graph convolutional networks with argument-aware
  pooling for event detection,'' in {\em Proceedings of the AAAI Conference on
  Artificial Intelligence}, vol.~32, 2018.

\bibitem{velivckovic2017graph}
P.~Veli{\v{c}}kovi{\'c}, G.~Cucurull, A.~Casanova, A.~Romero, P.~Lio, and
  Y.~Bengio, ``Graph attention networks,'' {\em arXiv preprint
  arXiv:1710.10903}, 2017.

\bibitem{wang2021deep}
K.~Wang, J.~Chen, Z.~Song, Y.~Wang, and C.~Yang, ``Deep neural network-embedded
  stochastic nonlinear state-space models and their applications to process
  monitoring,'' {\em IEEE Transactions on Neural Networks and Learning
  Systems}, 2021.

\bibitem{chen2020cone}
X.~Chen, M.~Heimann, F.~Vahedian, and D.~Koutra, ``Cone-align: Consistent
  network alignment with proximity-preserving node embedding,'' in {\em
  Proceedings of the 29th ACM International Conference on Information \&
  Knowledge Management}, pp.~1985--1988, 2020.

\bibitem{kingma2014adam}
D.~P. Kingma and J.~Ba, ``Adam: A method for stochastic optimization,'' {\em
  arXiv preprint arXiv:1412.6980}, 2014.

\bibitem{dinu2014improving}
G.~Dinu, A.~Lazaridou, and M.~Baroni, ``Improving zero-shot learning by
  mitigating the hubness problem,'' {\em arXiv preprint arXiv:1412.6568}, 2014.

\bibitem{hovcevar2014combinatorial}
T.~Ho{\v{c}}evar and J.~Dem{\v{s}}ar, ``A combinatorial approach to graphlet
  counting,'' {\em Bioinformatics}, vol.~30, no.~4, pp.~559--565, 2014.

\bibitem{hovcevar2016computation}
T.~Ho{\v{c}}evar and J.~Dem{\v{s}}ar, ``Computation of graphlet orbits for
  nodes and edges in sparse graphs,'' {\em Journal of Statistical Software},
  vol.~71, no.~1, pp.~1--24, 2016.

\bibitem{rossi2015network}
R.~Rossi and N.~Ahmed, ``The network data repository with interactive graph
  analytics and visualization,'' in {\em Twenty-Ninth AAAI Conference on
  Artificial Intelligence}, 2015.

\bibitem{amunts2013bigbrain}
K.~Amunts, C.~Lepage, L.~Borgeat, H.~Mohlberg, T.~Dickscheid, M.-{\'E}.
  Rousseau, S.~Bludau, P.-L. Bazin, L.~B. Lewis, A.-M. Oros-Peusquens, {\em
  et~al.}, ``Bigbrain: an ultrahigh-resolution 3d human brain model,'' {\em
  Science}, vol.~340, no.~6139, pp.~1472--1475, 2013.

\bibitem{iofciu2011identifying}
T.~Iofciu, P.~Fankhauser, F.~Abel, and K.~Bischoff, ``Identifying users across
  social tagging systems,'' in {\em Fifth International AAAI Conference on
  Weblogs and Social Media}, 2011.

\bibitem{van2008visualizing}
L.~Van~der Maaten and G.~Hinton, ``Visualizing data using t-sne.,'' {\em
  Journal of machine learning research}, vol.~9, no.~11, 2008.

\end{thebibliography}

\end{document}